\documentclass{article}

\usepackage{chngcntr} 
\usepackage{fullpage}
\usepackage{amsmath}
\usepackage{enumerate}
\usepackage{fancyhdr}
\usepackage{color}
\usepackage{listings}
\usepackage{graphicx}
\usepackage{amssymb}
\usepackage{mathtools}

\usepackage{algorithmic}
\usepackage[round]{natbib}
\definecolor{lightgray}{gray}{0.5}

\usepackage{comment}
\usepackage{epsf}
\usepackage{graphicx}

\usepackage{caption}
\usepackage{subcaption}

\usepackage{color}

\usepackage{mathtools}
\usepackage{amsfonts}
\usepackage{amsthm}
\usepackage{amsmath, bm}
\usepackage{amssymb}
\usepackage{url}
\usepackage{pgfplots}
\pgfplotsset{width=5.35cm,compat=1.9}

\usepackage{textcomp}
\usepackage{xcolor}

\DeclareSymbolFont{extraup}{U}{zavm}{m}{n}
\DeclareMathSymbol{\vardiamond}{\mathalpha}{extraup}{87}

\usepackage[ruled, vlined, lined, commentsnumbered]{algorithm2e}

\RestyleAlgo{ruled}

\newtheorem{theorem}{Theorem}

\newtheorem{proposition}{Proposition}
\newtheorem{lemma}{Lemma}
\newtheorem{corollary}{Corollary}
\newtheorem{definition}{Definition}
\newtheorem{remark}{Remark}

\long\def\comment#1{}

\newcommand{\poly}{\mathrm{poly}}

\newcommand{\E}{\ensuremath{{\mathbb{E}}}}

\DeclareMathOperator*{\argmax}{argmax}

\newcommand{\real}{\ensuremath{\mathbb{R}}}

\definecolor{lightgray}{gray}{0.5}

\title{\textbf{Adaptive Clustering and Personalization in  Multi-Agent Stochastic Linear Bandits}}
\author{Avishek Ghosh$^{\star, \dagger}$, \, Abishek Sankararaman$^{\star, \ddagger}$ and  Kannan Ramchandran$^\dagger$ \vspace{2mm} \\
Halıcıoğlu Data Science Institute (HDSI), UC San Diego$^\dagger$ \\
\vspace{1.5mm}
AWS AI, Palo Alto, USA$^\ddagger$  \\
email: a2ghosh$@$ucsd.edu, abisanka$@$amazon.com, \\ kannanr$@$eecs.berkeley.edu 
}

\begin{document}
\maketitle

\begin{abstract}
    It has been empirically observed in several recommendation systems, that their performance improve as more people join the system by learning \emph{across heterogeneous users}. In this paper, we seek to theoretically understand this phenomenon by studying the problem of minimizing regret in an $N$ users heterogeneous stochastic linear bandits framework. We study this problem under two models of heterogeneity; {\em (i)} a clustering framework where users are partitioned into groups with users in the same group being identical, but different across groups, and {\em (ii)} a personalization framework where no two users are necessarily identical, but are all similar. In the clustered users' setup,  we propose a successive refinement algorithm, which for any agent, achieves regret scaling as $\mathcal{O}(\sqrt{T/N})$, if the agent is in a `well separated' cluster, or scales as  $\mathcal{O}(T^{\frac{1}{2} + \varepsilon}/(N)^{\frac{1}{2} -\varepsilon})$ if its cluster is not well separated, where $\varepsilon$ is positive and arbitrarily close to $0$. In the personalization framework, we introduce a natural algorithm where, the personal bandit instances are initialized with the estimates of the global average model and show that, any agent $i$ whose parameter deviates from the population average by $\epsilon_i$, attains a regret scaling of $\widetilde{O}(\epsilon_i\sqrt{T})$. Our algorithms enjoy several attractive features of being \emph{problem complexity adaptive and parameter free} \textemdash if there is structure such as well separated clusters, or all users are similar to each other, then the regret of every agent goes down with $N$ (collaborative gain). On the other hand, in the worst case, the regret of any user is no worse than that of having individual algorithms per user that does not leverage collaborations. 
\end{abstract}

\section{INTRODUCTION}
\label{sec:intro}

Large scale web recommendation systems have become ubiquitous in the modern day, due to a myriad of applications that use them including online shopping services, video streaming services, news and article recommendations, restaurant recommendations etc, each of which are used by thousands, if not more users, across the world. For each user, these systems make repeated decisions under uncertainty, in order to better learn the preference of each individual user and serve them. A unique feature these large platforms have is that of \emph{collaborative learning} \textemdash namely applying the learning from one user to improve the performance on another \cite{rec_sys_collab_learning}. However, the sequential online setting renders this complex, as two users are seldom identical \citep{pintrest}. 

We study the problem of multi-user contextual bandits \citep{chatterji2020osom}, and quantify the gains obtained by collaborative learning under user heterogeneity. We propose two models of user-heterogeneity: (a) clustering framework where only users in the same group are identical (b) personalization framework where no two users are necessarily identical, but are close to the population average. 
Both these models are widely used in practical systems involving a large number of users (ex. \citep{pintrest, amazon, cluster_rec_2r, cluster_rec}). 
User clustering in such systems can be induced from a variety of factors such as affinity to similar interests, age-groups etc  \citep{embedding_1, word_embeddings_1, cold_start_embedding}. The personalization framework in these systems is also a natural in many neural network models, wherein users represented by learnt embedding vectors are not identical; nevertheless similar users are embedded nearby \citep{deep_matrix_factor_1, deep_matrix_factor_2, deep_embedding, embedding_1}.



Formally, our model consists of $N$ users, all part of a common platform. The interaction between the agents and platform proceeds in a sequence of rounds. Each round begins with the platform receiving $K$ contexts corresponding to $K$ items from the environment. The platform then recommends an item to each user and receives feedback from the users about the item. We posit that associated with each user $i$, is an preference vector $\theta^*_i$, initially unknown to the platform. In any round, the average reward (the feedback) received by agent $i$ for a recommendation of item, is the inner product of $\theta^*_i$ with the context vector of the recommended item. The goal of the platform is to maximize the reward collected over a time-horizon of $T$ rounds. Following standard terminology, we henceforth refer to an ``arm" and item interchangeably, and thus ``recommending item $k$" is synonymous to ``playing arm $k$". We also use agents and users interchangeably.

\paragraph{Example Application:} Our setting is motivated through a caricature of a news recommendation system serving $N$ users and $K$ publishers \citep{langford_news}. Each day, each of the $K$ publishers, publishes a news article, which corresponds to the context vector in our contextual bandit framework. In practice, one can use standard tools to embed articles in vector spaces, where the dimensions correspond to topics such as politics, religion, sports etc (\citep{doc_embedding}). The user preference indicates the interest of a user, and the reward, being computed as an inner product of the context vector and the user preference, models the observation that the more aligned an article is to a user's interest, the higher the reward.

For both  frameworks, we propose \emph{adaptive} algorithms; in the clustering setup, we propose Successive Clustering of Linear Bandits (SCLB), which is agnostic to the number of clusters, the gap between clusters and the cluster size. Yet SCLB yields regret that depends on these parameters, and is thus adaptive. In the personalization framework, our proposed algorithm, namely Personalized Multi-agent Linear Bandits (PMLB) adapts to the level of common representation across users. In particular, if an agents' preference vector is close to the population average, PMLB exploits that and incurs low regret for this agent due to collaboration. On the other hand if an agent's preference vector is far from the population average, PMLB yields a regret similar to that of OFUL \citep{chatterji2020osom} or Linear Bandit algorithms \citep{oful} that do not benefit from multi-agent collaboration. 

\vspace{-2mm}
\section{MAIN CONTRIBUTIONS} 
Our contributions are \textemdash {\em (i)} algorithmic and {\em (ii)} theoretical.
\vspace{-2mm}
\subsection{Algorithmic: Adaptive and Parameter-Free}

Our key novelty with respect to the algorithm is to propose adaptive and parameter free algorithms. Roughly speaking, an algorithm is parameter-free and adaptive, if does not need input about the difficulty of the problem, yet has regret guarantees that scale with the inherent complexity. In particular, we show in the two frameworks that, if there is structure, then the regret attained by our algorithms is much lower as they learn across users. Simultaneously, in the worst case, the regret guarantee is no worse than if every agent had its own algorithm without collaborations. 


\textbf{In the clustering framework}, we give a multi-phase, successive refinement based algorithm, SCLB,  which is parameter free---specifically no knowledge of cluster separation and number of clusters is needed. SCLB \emph{automatically} identifies whether a given problem instance is `hard' or `easy' and adapts to the corresponding regret. Concretely, SCLB attains per-agent regret $\mathcal{O}(\sqrt{T/N})$, if the agent is in a `well separated' (i.e. `easy') cluster, or   $\mathcal{O}(T^{\frac{1}{2} + \varepsilon}/(N)^{\frac{1}{2} -\varepsilon})$ if the agent's  cluster is not well separated (i.e., `hard'), where $\varepsilon$ is positive and arbitrarily close to $0$. \emph{This result holds true, even in the limit when the cluster separation approaches $0$}.  This shows that when the underlying instance gets harder to cluster, the regret is increased. Nevertheless, despite the clustering being hard to accomplish, every user still experiences collaborative gain of $N^{1/2 - \varepsilon}$ and regret sub-linear in $T$. Moreover, if clustering is easy i.e., well-separated, then the regret rate \emph{matches that of an oracle that knows the cluster identities}.

\textbf{In the personalization framework}, we give PMLB, a parameter free algorithm, whose regret adapts to an appropriately defined problem complexity -- if the users are similar, then the regret is low due to collaborative learning while, in the worst case, the regret is no worse than that of individual learning. Formally, we define the complexity as the \emph{factor of common representation}, which for agent $i$ is $\epsilon_i := \|\theta^*_i - \frac{1}{N}\sum_{l=1}^N \theta^*_l\|$, where $\theta^*_i \in \real^d$ is agent $i$'s  representation, and  $\frac{1}{N}\sum_{l=1}^N \theta^*_i$ is the average representation of $N$ agents. PMLB adapts to $\epsilon_i$ gracefully (without knowing it apriori) and yields a regret of $\mathcal{O}(\epsilon_i \sqrt{dT})$. Hence, if the agents share representations, i.e., $\epsilon_i$ is small, then PMLB obtains low regret. On the other hand, if $\epsilon_i$ is large, say $\mathcal{O}(1)$, the agents do not share a common representation, the regret of PMLB is $\mathcal{O}(\sqrt{dT})$, which matches that obtained by each agent playing OFUL, independently of other agents. Thus, PMLB benefits from collaborative learning and obtains small regret, if the problem structure admits, else the regret matches the baseline strategy of every agent running an independent bandit instance.

\textbf{Empirical Validation:} 
We empirically verify the theoretical insights on both synthetic and Last.FM real data.  
In the clustering framework, we compare with three benchmarks \textemdash CLUB \citep{clustering_online}, SCLUB \citep{set_club}, and a simple baseline where every agent runs an independent bandit model, i.e.,  no collaboration. We observe that our algorithms have superior performance compared to the benchmarks in a variety of settings. We observe similar performance in the personalization framework also. 

\subsection{Theoretical: Improved bounds for Clustering}
It is worth pointing out that SCLB works for \emph{all} ranges of separation, which is starkly different from standard algorithms in bandit clustering (\citep{clustering_online,gentile2017context,korda2016distributed}) and statistics (\citep{em_2,em_1}). We now compare our results to CLUB \citep{clustering_online}, that can be modified to be applicable to our setting (c.f. Section \ref{sec:simulations}) (note that we make \emph{identical assumptions} to that of CLUB). First, CLUB is non-adaptive and its regret guarantees hold only when the clusters are separated. Second, even in the separated setting, the separation (gap) cannot be lower than $\mathcal{O}(1/T^{1/4})$ for CLUB, while it can be as low as $\mathcal{O}(1/T^\alpha)$, where $\alpha<1/2$ for SCLB. Moreover, in simulations (Section~\ref{sec:simulations}) we observe that SCLB outperforms CLUB in a variety of synthetic and a real data setting.

The key innovations we introduce in the analysis are that of \emph{`perturbed OFUL'} and the \emph{`shifted OFUL'} algorithms in the clustering and personalization setup respectively. In the clustering setup, our algorithm first runs individual OFUL instances per agent, estimates the parameter, then clusters the agents and treats all agents of a single cluster as one entity. In order to prove that this works even when the cluster separation is small, we need to analyze the behaviour of OFUL where the rewards come from a slightly perturbed model. In the personalization setup, our algorithm first estimates the mean vector $\Bar{\theta^*} := \frac{1}{N}\sum_{i=1}^N\theta^{*}_i$ of the population. Subsequently, the algorithm subtracts the effect of the mean and only learns the component $\theta^{*}_i - \Bar{\theta^{*}}$ by compensating the rewards. Our technical innovation is to show that with high probability, shifting the rewards by any fixed vector can only increase overall regret (Lemma \ref{lem:shift-oful}).

\section{RELATED WORK}
Collaborative gains in multi-user recommendation systems have long been studied in Information retrieval and recommendation systems (ex. \citep{cluster_rec, cluster_rec_2r, amazon,rec_sys_collab_learning}). The focus has been in developing effective ideas to help practitioners deploy large scale systems. Empirical studies of recommendation system has seen renewed interest lately due to the integration of deep learning techniques with classical ideas (ex. \citep{murali, zhao2019recommending, deep_learn_rec, covington2016deep, deep_embedding, rec_deep}). Motivated by the empirical success, we undertake a theoretical approach to quantify collaborative gains achievable in a contextual bandit setting.
Contextual bandits has proven to be fruitful in modeling sequential decision making in many applications \citep{langford_news, gang_bandits, clustering_online}.  

The paper of \citep{clustering_online} is closest to our clustering setup, where in each round, the platform plays an arm for a single randomly chosen user. As outlined before, our algorithm obtains a superior performance, both in theory and empirically. For personalization, the recent paper of \citep{yang2021impact} is the closest, which posits all users's parameters to be in a common low dimensional subspace. \citep{yang2021impact} proposes a learning algorithm under this assumption. In contrast, we make no parametric assumptions, and demonstrate an algorithm that achieves collaboration gain, if there is structure, while degrading gracefully to the simple baseline of independent bandit algorithms in the absence of structure.

The framework of personalized learning has been exploited in a great detail in representation learning and meta-learning. While \citep{d2019sharing,lazaric,rusu2015policy,higgins2017darla,parisotto2015actor} learn common representation across agents in Reinforcement Learning, \citep{arora2020provable} uses it for imitation learning. 
We remark that representation learning is also closely connected to meta-learning \citep{meta,finn2019online,khodak2019adaptive}, where close but a common initialization is learnt from leveraging non identical but similar representations. Furthermore, in Federated learning, the problem of personalization is a well studied problem \citep{three-google,fallah2020personalized,fallah2020convergence}.



\section{PROBLEM SETUP}
\label{sec:setup}

\textbf{Users and Arms}: Our system consists of $N$ users, interacting with a centralized system (termed as `center' henceforth) repeatedly over $T$ rounds. At the beginning of each round, environment provides the center with $K$ context vectors corresponding to $K$ arms, and for each user, the center recommends one of the $K$ arms to play. At the end of the round, every user receives a reward for the arm played, which is observed by the center. The $K$ context vectors in round $t$ are denoted by $\beta_t = [\beta_{1,t},\ldots,\beta_{K,t}]\in \real^{d \times K}$.

\textbf{User heterogeneity:} Each user $i$, is associated with a preference vector $\theta^*_i \in \mathbb{R}^d$, and the reward user $i$ obtains from playing arm $j$ at time $t$ is 
is given by $\langle \beta_{j,t}, \theta^* \rangle + \xi_t$. Thus, the structure of the set of user representations $(\theta^*_i)_{i=1}^N$ govern how much benefit from collaboration can be expected. In the rest of the paper, we consider two instantiations of the setup - a clustering framework and the personalization framework.

\textbf{Stochastic Assumptions}: We follow the framework of \citep{oful, chatterji2020osom} and assume that $(\xi_t)_{t \geq 1}$ and $(\beta_t)_{t \geq 1}$ are random variables. We denote by $\mathcal{F}_{t-1}$, as the sigma algebra generated by all noise random variables upto and including time $t-1$.
 We denote by $\mathbb{E}_{t-1}(.)$ and $\mathbb{V}_{t-1}(.)$ as the conditional expectation and conditional variance operators respectively with respect to $\mathcal{F}_{t-1}$. 
We assume that the $(\xi_t)_{t \geq 1}$ are conditionally sub-Gaussian noise with known parameter $\sigma$, conditioned on all the arm choices and realized rewards in the system upto and including time $t-1$. Without loss of generality, we assume $\sigma=1$ throughout. The contexts $\beta_{i,t} \in \mathbb{B}^d_2(1)$ are assumed to be drawn independent of both the past
and $\{\beta_{j,t}\}_{j \neq i}$, satsifying
\begin{align}
\label{eqn:context}
    \mathbb{E}_{t-1}[\beta_{i,t}] = 0 \qquad \mathbb{E}_{t-1} [\beta_{i,t} \, \beta_{i,t}^\top] \succeq \rho_{\min} I.
\end{align}
Moreover, for any fixed $z \in \mathbb{R}^d$, of unity norm, the random variable $(z^\top \beta_{i,t})^2$ is conditionally sub-Gaussian, for all $i$, with $\mathbb{V}_{t-1}[(z^\top \beta_{i,t})^2)] \leq 4 \rho_{\min}$.
This means that the conditional mean of the covariance matrix is zero and the conditional covariance matrix is positive definite with minimum eigenvalue at least $\rho_{\min}$. 

Furthermore, the conditional variance assumption is crucially required to apply \eqref{eqn:context} for contexts of (random) bandit arms selected by our learning algorithm (see \citep[Lemma 1]{clustering_online}). Note this this set of assumptions is not new and the exact set of assumptions were used in \citep{clustering_online,chatterji2020osom}\footnote{The conditional variance assumption is implicitly used in \citep{chatterji2020osom} without explicit statement.} for online clustering and binary model selection respectively. Furthermore, \citep{foster2019model} uses similar assumptions for stochastic linear bandits and \citep{ghosh2021model} uses it for model selection in Reinforcement learning problems with function approximation. Also, since the context vectors are drawn from unit sphere (and hence sub-Gaussian), we have $\rho_{\min} \leq 1/d$, and hence one needs to track the dependence on $\rho_{\min}$. Observe that our stochastic assumption also includes the simple setting where the contexts evolve according to a random process independent of the actions and rewards from the learning algorithm.

\textbf{Performance Metric:} At time $t$, we denote by $B_{i,t} \in [K]$ to be the arm played by any agent $i$ with preference vector $\theta^*_i$. The corresponding regret, over a time horizon of $T$ is given by
\begin{align}
\label{eqn:regret_def}
   R_i(T) = \sum_{t=1}^T \mathbb{E}  \max_{j \in [K]} \langle \theta^*_i, \beta_{j,t} - \beta_{B_{i,t},t} \rangle  
\end{align}

Throughout, OFUL refers to the linear bandit algorithm of \citep{oful}, which we use as a blackbox. In particular we use a variant of the OFUL as prescribed in \citep{chatterji2020osom}\footnote{We use OFUL as used in the OSOM algorithm of  \citep{chatterji2020osom} without bias for the linear contextual setting. }.

\section{CLUSTERING FRAMEWORK}
\label{sec:clustering}
\begin{algorithm}[t!]
  \caption{Successive Clustering of Linear Bandits (SCLB) }
  \begin{algorithmic}[1]
 \STATE  \textbf{Input:} No. of users $N$, horizon $T$, parameter $\alpha <1/2$, constant $C$, high probability bound $\delta$
 \FOR {phases $1 \leq i \leq \log_2(T)$}
 \STATE Play CMLB ($\gamma = 3/(N2^i)^{\alpha}$, horizon $T = 2^i$, high probability $\delta/2^i$, cluster-size $p^{*} = i^{-2}$)
 \ENDFOR
   \end{algorithmic}
  \label{algo:size_unknown}
\end{algorithm}
We assume that the users' vectors $(\theta^*_i)_{i=1}^N$ are clustered into $L$ groups, with $p_i \in (0,1]$ denoting the fraction of users in cluster $i$. All users in the same cluster have the same context vector, and thus without loss of generality, for all clusters $i \in [L]$, we denote by $\theta^*_i$, to be the preference vector of any user of cluster $i$. We define \emph{separation parameter}, or SNR (signal to noise ratio) of cluster $i$ as $\Delta_i := \min_{j \in [L]\setminus\{ i\}} \|\theta^*_i - \theta^*_j \|$, smallest distance to another cluster.


\paragraph{Learning Algorithm:} We propose the Successive Clustering of Linear Bandits (SCLB) algorithm in Algorithm \ref{algo:size_unknown}. SCLB does not need any knowledge of the gap $\{\Delta_i\}_{i=1}^L$, the number of clusters $L$ or the cluster size fractions $\{p_i\}_{i=1}^L$. Nevertheless, SCLB adapts to the problem SNR and yields regret accordingly. One attractive feature of Algorithm \ref{algo:size_unknown} is that it works uniformly \emph{for all} ranges of the gap $\{\Delta_i\}_{i=1}^L$. This is in sharp contrast with the existing algorithms \citep{clustering_online} which is only guaranteed to give good performance when the gap $\{\Delta_i\}_{i=1}^L$ are large enough. Furthermore, our uniform guarantees are in contrast with the works in standard clustering algorithms, where theoretical guarantees are only given for a sufficiently large separation \citep{em_1,em_2}.

\begin{algorithm}[t!]
\caption{Clustered Multi-Agent Bandits (CMLB)}
  \begin{algorithmic}[1]
    \STATE \textbf{Input:} No. of users $N$, horizon $T$, parameter $\alpha <1/2$, constant $C$, high probability bound $\delta$, threshold $\gamma$,  cluster-size parameter $p^{*}$ \\
    \vspace{1mm}
  {\textbf{ Individual Learning Phase}}    
 \STATE   $T_{\text{Explore}} \gets C^{(2)} d (NT)^{2\alpha} \log(1/\delta)$ 
\STATE All agents play OFUL($\delta)$ independently for $T_{\text{explore}}$ rounds
    \STATE $\{\hat{\theta}^{(i)}\}_{i=1}^N \gets $ All agents' estimates at the end of round $T_{\text{explore}}$.\\
{\textbf{ Cluster the Users}}    
\STATE {\ttfamily User-Clusters} $\gets $ {\ttfamily MAXIMAL-CLUSTER$(\{\hat{\theta}^{(i)}\}_{i=1}^N, \gamma$, $p^{*})$} \\
\vspace{1mm} {\textbf{ Collaborative Learning Phase}}    
  \STATE Initialize one OFUL($\delta$) instance per-cluster
 \FOR {clusters $\ell \in \{1,\ldots,| \text{{\ttfamily User-Clusters}} |\}$ \textbf{in parallel} }
 \FOR {times $t \in \{T_{\text{explore}}+1,\cdots,T\}$}
 \STATE All users in the $\ell$-th cluster play the arm given by the OFUL algorithm of cluster $l$.
 \STATE Average of the observed rewards of all users of cluster $l$ is used to update the OFUL($\delta$) state of cluster $l$
 \ENDFOR
 \ENDFOR
  \end{algorithmic}
  \label{alg:CMLB}
\end{algorithm}

SCLB is a multi-phase algorithm, which invokes Clustered Multi-agent Linear Bandits (CMLB) (Algorithm~\ref{alg:CMLB}) repeatedly, by decreasing the size parameter, namely $p^{*}$ polynomially and high probability parameter $\delta_i$ exponentially. Algorithm \ref{algo:size_unknown} proceeds in phases of exponentially growing phase length with phase $j \in \mathbb{N}$ lasting for $2^j$ rounds. In each phase, a fresh instance of CMLB is instantiated with high probability parameter $\delta/2^j$ and the minimum size parameter $j^{-2}$. Thus, as the phase length grows, the size parameter sent as input to Algorithm \ref{alg:CMLB} decays. We show that this simple strategy suffices to show that the size parameter converges to $p_i$, and we obtain collaborative gains without knowledge of $p_i$. 

\textbf{CMLB (Algorithm~\ref{alg:CMLB}) :} CMLB works in the three phases: (a) (Individual Learning) the $N$ users play an independent linear bandit algorithm to (roughly) learn their preference; (b) (Clustering)  users are clustered based on their estimates using \texttt{MAXIMAL CLUSTER} (Algorithm~\ref{algo:clustering}); and (c) (Collaborative Learning) one Linear Bandit instance per cluster is initialized and all users of a cluster play the same arm. The average reward over all users in the cluster is used to update the per-cluster bandit instance. When clustered correctly, the learning is faster, as the noise variance is reduced due to averaging across users.
Note that \texttt{MAXIMAL CLUSTER} algorithm requires a size parameter $p^*$. 
\subsection{Regret guarantee of SCLB}
As mentioned earlier, SCLB is an adaptive algorithm that yields provable regret for \emph{all ranges} of $\{\Delta_i\}_{i=1}^L$. When $\{\Delta_i\}_{i=1}^L$ are large, SCLB can cluster the agents perfectly, and thereafter exploit the collaborative gains across users in same cluster. On the other hand, if $\{\Delta_i\}_{i=1}^L$ are small, SCLB still adapts to the gap, and yields a non-trivial (but sub-optimal) regret. As a special case, we show that if all the clusters are very close to one another, then with high probability, SCLB  identifies treats all agents as \emph{one big} cluster, yielding highest collaborative gain. 
 
Without loss of generality, in what follows, we focus on an arbitrary agent belonging to cluster $i$ and characterize her regret. Throughout this section, we assume
 \begin{align}
 \label{eqn:tau}
     T &\geq C \frac{1}{N}\bigg[\frac{\tau_{\min}(\delta) \rho_{\min}}{d \log(1/\delta)} \bigg]^{\frac{1}{2\alpha}}, \,\,\,  \\ \,\,\, \tau_{\min}(\delta) &= \bigg[\frac{16}{\rho_{\min}^2}+\frac{8}{3\rho_{\min}} \bigg] \log(\frac{2dT}{\delta})
 \end{align}
 \begin{definition}[$\alpha$-Separable Cluster]
 For a fixed $\alpha <1/2$, cluster $i \in [L]$ is termed $\alpha$-separable if $\Delta_i \geq \frac{5}{(NT)^\alpha}$. Otherwise, it is termed as $\alpha$-inseparable.
 \end{definition}
 

\begin{lemma}
\label{lem:separated_cluster}
 If CMLB is run with parameters $\gamma = 3/(NT)^\alpha$ and $p^* \leq p_i$ and $\alpha < \frac{1}{2}$, then with probability at least $1-2 \,\binom{N}{2} \delta$, any cluster $i$ that is $\alpha$-separable is clustered correctly. Furthermore, the regret of any user in the $\alpha$-separated cluster $i$ satisfies,
 \begin{align*}
    R_i(T) \leq C_1   \left [\frac{d}{\rho_{\min}}  (NT)^\alpha +  \sqrt{\frac{d}{\rho_{\min}} } (\sqrt{\frac{T-\frac{d (NT)^{2\alpha}}{\rho_{\min}} \log(1/\delta)}{p_i N}} ) \right ] \log(1/\delta),
     \end{align*}
with probability exceeding $1-4\binom{N}{2} \delta$.
\end{lemma}
We now present the regret of SCLB for the setting with separable cluster
\begin{theorem}
If Algorithm \ref{algo:size_unknown} is run for $T$ steps with parameter $\alpha < \frac{1}{2}$, then the regret of any agent in a cluster $i$ that is $\alpha$-separated satisfies
\begin{align*}
    R_i(T) \leq 4\left(2^{\frac{1}{\sqrt{p_i}}} \right) +  C_2  \bigg[ \frac{d}{\rho_{\min}} (NT)^{\alpha}    + \sqrt{\frac{dT}{\rho_{\min}{N}}} \bigg]  \log^2(T) \log (1/\delta),
\end{align*}
with probability at-least $1-cN^2\delta$. Moreover, if
{\tiny $\alpha \leq \frac{1}{2} ( \frac{\log\left[\frac{\rho_{\min}T}{d p_i N}\right]}{\log(NT)})$}, we have $R_i(T) \leq \Tilde{\mathcal{O}} [ 2^{\frac{1}{\sqrt{p_i}}} + \sqrt{\frac{d}{\rho_{\min}}} \,\, \sqrt{\frac{T}{N}} ] \log(1/\delta)$.
\label{thm:sclb_separated}
\end{theorem}
\begin{algorithm}[t!]
  \caption{{\ttfamily MAXIMAL-CLUSTER}}
  \begin{algorithmic}[1]
 \STATE  \textbf{Input:} All estimates $\{\hat{\theta}^{(i)}\}_{i=1}^N$, size parameter $p^{*} >0$, threshold $\gamma \geq 0$.
 \STATE Construct an undirected Graph $G$ on $N$ vertices as follows:
$
 || \widehat{\theta}^*_i - \widehat{\theta}^*_j || \leq \gamma \Leftrightarrow  i \sim_G j
$
 \STATE $\mathcal{C} \gets \{C_1, \cdots, C_k \} $ all the connected components of $G$
 \STATE $\mathcal{S}(p^*) \gets \{ C_j : |C_j| < p^{*}N \}$ \COMMENT {All Components smaller than $p^{*}N$}
 \STATE $C^{(p)} \gets \cup_{C \in \mathcal{S}(p^{*})} C$ \COMMENT {Collapse all small components into one}
 \STATE \textbf{Return :} $\mathcal{C}\setminus \mathcal{S}(p^*) \bigcup C^{(p)}$ \COMMENT {Each connected component larger than $p^*N$ is a cluster, and all small components are a single cluster}
  \end{algorithmic}
  \label{algo:clustering}
\end{algorithm}
\begin{remark}
Note that we obtain the regret scaling of $\Tilde{\mathcal{O}}(\sqrt{T/N})$, which is optimal, i.e., the regret rate matches an oracle that knows cluster membership.  The cost of successive clustering is $\mathcal{O}(2^{\frac{1}{\sqrt{p_i}}})$, which is a $T$-independent (problem dependent) constant.
\end{remark}
\begin{remark}
Note that the separation we need is only $5/(NT)^\alpha$. This is a weak condition since in a collaborative system with large $N$ and $T$, this quantity is sufficiently small.
\end{remark}
\begin{remark}
Observe that $R_i(T)$ is a decreasing function of $N$. Hence, more users in the system ensures that the regret decreases. This is collaborative gain.
\end{remark}
\begin{remark}
(Comparison with \citep{clustering_online}) Note that in a setup where clusters are separated, \citep{clustering_online} also yields a regret of $\Tilde{\mathcal{O}}(\sqrt{T/N})$. However, the separation between the parameters (gap) for \citep{clustering_online} cannot be lower than $\mathcal{O}(1/T^{1/4})$, in order to maintain order-wise optimal regret. On the other hand, we can handle separations of the order $\mathcal{O}(1/T^\alpha)$, and since $\alpha <1/2$, this is a strict improvement over \citep{clustering_online}. 
\end{remark}
\begin{remark}
The constant term $\mathcal{O}(2^{\frac{1}{\sqrt{p_i}}})$ can be removed if we have an estimate of the $p_i$. Here, instead of SCLB, we simply run CMLB with the estimate of $p_i$ and obtain the regret of Lemma~\ref{lem:separated_cluster}, without the term $\mathcal{O}(2^{\frac{1}{\sqrt{p_i}}})$. Note that in simulations (Sec.~\ref{sec:simulations}), we observe that the size input to CMLB is not needed.
\end{remark}
We now present our results when cluster $i$ is $\alpha$-inseparable.
\begin{lemma}
\label{lem:inseparable_clusters}
If CMLB  is run with input $\gamma = 3/(NT)^\alpha$ and $p^* \leq p_i$ and $\alpha < \frac{1}{2}$, then any 
 user in a cluster $i$ that is $\alpha$-inseparable satisfies
 \begin{align*}
    R(T) \leq C_1 L (\frac{T^{1-\alpha}}{N^\alpha}) +  C_2 \sqrt{\frac{d}{\rho_{\min}} } \,\, [\sqrt{\frac{T-\frac{d (NT)^{2\alpha}}{\rho_{\min}} \log(1/\delta)}{p^* N}} ] \log(1/\delta),
    \end{align*}
with probability at least $1- 4 \binom{N}{2} \delta$.
\end{lemma}
\begin{theorem}
\label{thm:inseparable}
If Algorithm \ref{algo:size_unknown} is run for $T$ steps with parameter $\alpha < \frac{1}{2}$, then the regret of any agent in a cluster $i$ that is $\alpha$-inseparable satisfies
\begin{align*}
    R_i(T) \leq 4(2^{\frac{1}{\sqrt{p_i}}})+ C \, L  (\frac{T^{1-\alpha}}{N^{\alpha}}) \log(T) + C_1  \sqrt{\frac{dT}{N\rho_{\min}}}\log(1/\delta) \ \log^2(T),
\end{align*}
with probability at-least $1-cN^2\delta$. Moreover, if
If $\alpha = \frac{1}{2}- \varepsilon$, where $\varepsilon$ is a positive constant arbitrarily close to $0$, we obtain, $
    R(T) \leq \Tilde{\mathcal{O}} \bigg [ 2^{\frac{1}{\sqrt{p_i}}} + L ( \frac{T^{\frac{1}{2}+\varepsilon}}{N^{\frac{1}{2}-\varepsilon}} ) + \sqrt{\frac{d}{\rho_{\min}} } \,\, (\sqrt{\frac{T}{N}} ) \log(1/\delta) \bigg].
$
\end{theorem}
\begin{remark}
As $\varepsilon >0$, the regret scaling of $\Tilde{\mathcal{O}}( \frac{T^{\frac{1}{2}+\varepsilon}}{N^{\frac{1}{2}-\varepsilon}} )$ is strictly worse than the optimal rate of $\Tilde{\mathcal{O}}(\sqrt{T/N})$. This can be attributed to the fact that the gap (or SNR) can be arbitrarily close to $0$, and inseparability of the clusters makes the problem harder to address.
\end{remark}
\begin{remark}
In this setting of low gap (or SNR), where the clusters are inseparable, most existing algorithms (for example \citep{clustering_online}) are not applicable. However, we still manage to obtain sub-optimal but non-trivial regret with high probability.
\end{remark}
\vspace{-2mm}
\paragraph{Special case of all clusters being close}
If $\max_{i\neq j}\|\theta^*_i -\theta^*_j \| \leq 1/(NT)^\alpha$, CMLB puts all the users in one big cluster. The collaborative gain in this setting is the largest. Here the regret guarantee of SCLB will be similar to that of Theorem~\ref{thm:inseparable} with $p_i =1$. We defer to Appendix~\ref{sec:all_club} for a detailed analysis.

\begin{remark}
Observe that if all agents are identical $\max_{i\neq j}\|\theta^*_i -\theta^*_j \| =0$  our regret bound does not match that of an \emph{oracle} which knows such information. The oracle guarantee would be $\mathcal{O}(\sqrt{T/N})$, whereas our guarantee is strictly worse. The additional regret stems from the universality of our algorithm as it works for all ranges of $\Delta_i$. 
\end{remark}

\section{PERSONALIZATION}
\label{sec:personalization}

In this section, we assume that the users' representations $\{\theta^*_i\}_{i=1}^N$ are similar but not necessarily identical. Of course, without any structural similarity among $\{\theta^*_i\}_{i=1}^N$, the only way-out is to learn the parameters separately for each user. In the setup of personalized learning, it is typically assumed that (see \citep{yang2021impact, collins2021exploiting, NEURIPS2020_24389bfe,ditto-virginia} and the references therein) that the parameters $\{\theta^*_i\}_{i=1}^N$ share some commonality, and the job is to learn the shared components or representations of $\{\theta^*_i\}_{i=1}^N$  collaboratively. After learning the common part, the individual representations can be learnt locally at each agent.

Similar to Section~\ref{sec:setup}, the contexts $\beta_{i,t}$-s are drawn independent of the past from a distribution such that $\beta_{i,t}$ is independent of $\{\beta_{j,t}\}_{j \neq i}$. 
In this section, we assume that $\beta_{i,t}$-s are drawn from~$\mathsf{Unif}[-1/\sqrt{d},1/\sqrt{d}]^{\otimes d}$ instead of the unit ball $\mathbb{B}_d^{(1)}$. 
The choice of uniform distribution is for simplicity, in general any distribution supported on $[-c,c]^{\otimes d}$ suffices. The scaling of $1/\sqrt{d}$ ensures that the norm is $\mathcal{O}(1)$. With this, the conditions of equation~\eqref{eqn:context} are satisfied with $\rho_{\min} = c_0/d$ (where $c_0$ is a constant) using \citep{vershynin2011introduction}. This is without loss of generality, as it simplifies the exposition. 

\begin{algorithm}[t!]
  \caption{Personalized Multi-agent Linear Bandits (PMLB)}
  \begin{algorithmic}[1]
 \STATE  \textbf{Input:} Agents $N$, Horizon $T$ \\ \vspace{1mm}
 \textbf{Common representation learning : Estimate $\Bar{\theta}^*=\frac{1}{N}\sum_{i=1}^N \theta^*_i$} \\
 \vspace{1mm}
 \STATE Initialize a single instance of OFUL($\delta$), called common OFUL 
 \FOR {times $t \in \{1,\cdots,\sqrt{T}\}$}
 \STATE All agents play the action given by the common OFUL
 \STATE Common OFUL's state is updated by the average of observed rewards at all agents
 \ENDFOR
 \STATE $\widehat{\theta}^* \gets $ the parameter estimate of Common OFUL at the end of round $\sqrt{T}$\\ \vspace{1mm}
 \textbf{Personal Learning} \\ \vspace{1mm}
 \FOR {agents $i \in \{1,\ldots,N\}$ \textbf{in parallel}}
 \STATE Initialize one ALB-Norm($\delta$) of \citep{ghosh_adaptive} instance per agent (reproduced in Algorithm \ref{algo:alb_norm} in Supplementary Material)
 \FOR {times $t \in \{\sqrt{T}+1,\ldots,T\}$}
 \STATE Agents play arm output by their personal copy of ALB-Norm (denoted as $\beta_{b_t^{(i)},t}$) and receive reward $y_t$
 \STATE Every agent updates their ALB-Norm state with corrected reward $ \Tilde{y}^{(t)}_i = y^{(t)}_i -\langle \beta_{b_t^{(i)},t}, \Hat{\theta}^* \rangle $
 
 \ENDFOR
 \ENDFOR
  \end{algorithmic}
  \label{algo:personal}
\end{algorithm}

We now define the notion of common representation across users. For consistency, similar to Section~\ref{sec:setup}, we assume $\| \theta^*_l \| \leq 1$ for all $l\in [N]$. We define $\Bar{\theta}^* = \frac{1}{N}\sum_{l=1}^N \theta^*_l$ as the average parameter.

\begin{definition}
($\epsilon$ common representation) An agent $i$ has $\epsilon_i$ common representation across $N$ agents if \, $\|\theta^*_i - \Bar{\theta}^* \| \leq \epsilon_i$, where $\epsilon_i$ is defined as the common representation factor.
\end{definition}
The above definition characterizes how far the representation of agent $i$ is from the average representation $\Bar{\theta}^*$. Note that since $\|\theta^*_l\| \leq 1$ for all $l$, we have $\epsilon_i \leq 2$. Furthermore, if $\epsilon_i$ is small, one can hope to exploit the common representation across users. On the other hand, if $\epsilon_i$ is large (say $\mathcal{O}(1)$), there is no hope to leverage collaboration across agents.

\subsection{The PMLB Algorithm}

Algorithm~\ref{algo:personal} works in two phases \textemdash {\em (i)} a common representation learning and {\em (ii)} a personal fine-tuning.

\textbf{Common Representation Learning:} In the first phase, PMLB learns the average representation $\Bar{\theta}^*$ by recommending the same arm to all users and averaging the obtained rewards. At the end of this phase, the center has the estimate $\Hat{\theta}^*$ of the average representation $\Bar{\theta}^*$. Since the algorithm aggregates the reward from all $N$ agents, it turns out that the common representation learning phase can be restricted to $\sqrt{T}$ steps.


\textbf{Personal Fine-tuning} In the personal learning phase, the center learns the vector $\theta^*_i - \Hat{\theta}^*$, \emph{independently} for every agent. For learning $\theta^*_i - \Hat{\theta}^*$, we employ the Adaptive Linear Bandits-norm (\texttt{ALB-norm}) algorithm of \citep{ghosh_adaptive}. \texttt{ALB-norm} is adaptive, yielding a norm dependent regret, i.e., depends on $\|\theta^*_i - \Hat{\theta}^*\|$. The idea here is to exploit the fact that in the common learning phase we have a good estimate of $\Bar{\theta}^*$. Hence, if the common representation factor $\epsilon_i$ is small, then $\|\theta^*_i - \Hat{\theta}^*\|$ is small, and it reflects in the regret expression. In order to estimate the difference, the center \emph{shifts} the reward by the inner product of the estimate $\Hat{\theta}^*$. By exploiting the anti-concentration property of Chi-squared distribution along with some standard results from optimization, we show that the regret of the shifted system is worse than the regret of agent $i$ (both in expectation and in high probability)\footnote{This is intuitive since, otherwise one can find \emph{appropriate shifts} to reduce the regret of OFUL, which contradicts the optimality of OFUL.}.


\subsection{Regret Guarantee for PMLB}



\begin{theorem}
\label{thm:personal}
Playing Algorithm ~\ref{algo:personal} with $T$ time and $\delta$, where $T \geq \tau_{\min}^2(\delta)$ ($\tau_{\min}(\delta)$ is defined in eqn.~\eqref{eqn:tau}) and $d \geq C \log (K^2 T)$, then the regret of agent $i$ satisfies
\begin{align*}
    R_i(T) \leq \Tilde{\mathcal{O}} ( \epsilon_i\, \sqrt{dT}  + \,\, T^{1/4}\,\, \sqrt{\frac{d^2}{\rho_{\min} N}}) \log^2(1/\delta),
\end{align*}
with probability at least $1-c\delta - \frac{1}{\poly(T)}$.
\end{theorem}
\begin{remark}
The leading term in regret is $\Tilde{\mathcal{O}}(\epsilon_i \sqrt{dT})$. If the common representation factor $\epsilon_i$ is small, PMLB exploits that across agents and as a result the regret is small as well. 
\end{remark}
\begin{remark}
Moreover, if $\epsilon_i$ is big enough, say $\mathcal{O}(1)$, this implies that there is no common representation across users, and hence collaborative learning is meaning less. In this case, the agents learn individually (by running OFUL), and obtain a regret of $\Tilde{\mathcal{O}}(\sqrt{dT})$ with high probability. Note that this is being reflected in Theorem~\ref{thm:personal}, as the regret is $\Tilde{\mathcal{O}}(\sqrt{dT})$, when $\epsilon_i = \mathcal{O}(1)$.
\end{remark}
The above remarks imply the adaptivity of PMLB. Without knowing the common representation factor $\epsilon_i$, PMLB indeed adapts to it---meaning that yields a regret that depends on $\epsilon_i$. If $\epsilon_i$ is small, PMLB leverages common representation learning across agents, otherwise when $\epsilon_i$ is large, it yields a performance equivalent to the individual learning. Note that this is intuitive since with high $\epsilon_i$, the agents share no common representation, and so we do not get a regret improvement in this case by exploiting the actions of other agents.

\begin{remark} (Lower Bound)
When $\epsilon_i = 0$, i.e., in the case when all agents have the identical vectors $\theta^*_i$, then Theorem \ref{thm:personal} gives a regret scaling as $R_i(T) \leq \widetilde{\mathcal{O}}( T^{1/4}d\sqrt{\frac{1}{\rho_{\min}N}})$. When the contexts are adversarily generated, \cite{chu2011contextual} obtain a lower bound (in expectation) of $\Omega(\sqrt{d T})$. However, in the presence of stochastic context, a lower bound on the contextual bandit problem is unknown to the best of our knowledge.
\label{remark:zero_epsilon}
\end{remark}

The requirement on $d$ in Theorem~\ref{thm:personal} can be removed if we consider the expected regret.
\begin{corollary}
\label{cor:exp-regret}
(Expected Regret) Suppose $T \geq \tau_{\min}^2(\delta)$ for $\delta>0$. The expected regret of the $i$-th agent after running Algorithm~\ref{algo:personal} for $T$ time steps is given by
\begin{align*}
    \mathbb{E}[R_i(T)] \leq \Tilde{\mathcal{O}} (\epsilon_i\, \sqrt{dT}  + \,\, T^{1/4}\,\, \sqrt{\frac{d^2}{\rho_{\min} N}}).
\end{align*}
\end{corollary}

\section{SIMULATIONS}
\label{sec:simulations}
In this section, we validate our theoretical finding with synthetic as well as real data.
\begin{figure}[t!]
\label{fig:synthetic}
\centering
    \subfloat[][]{
    \includegraphics[width=0.33\linewidth]{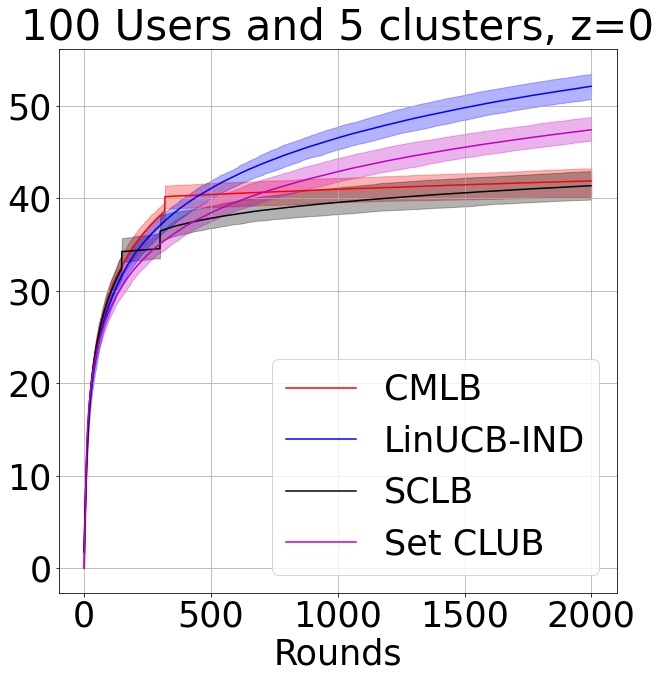}}
    \subfloat[][]{
    \includegraphics[width=0.33\linewidth]{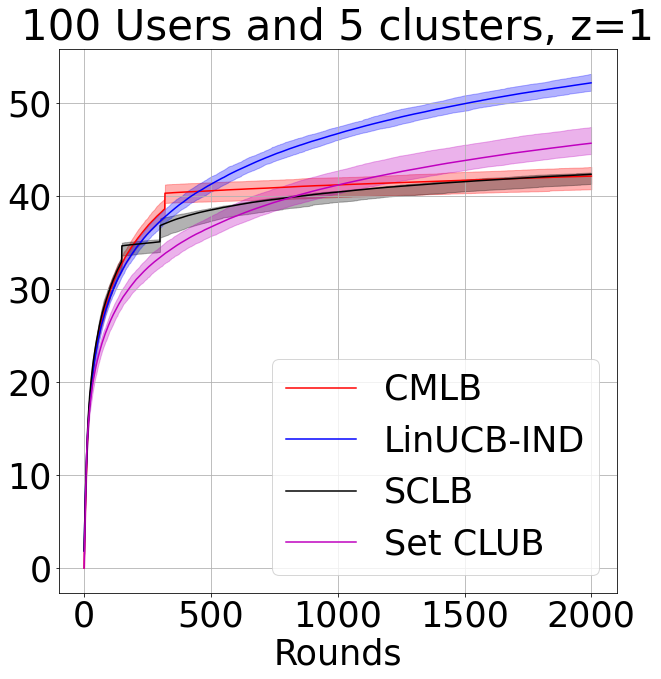}}
    \subfloat[][]{
    \includegraphics[width=0.33\linewidth]{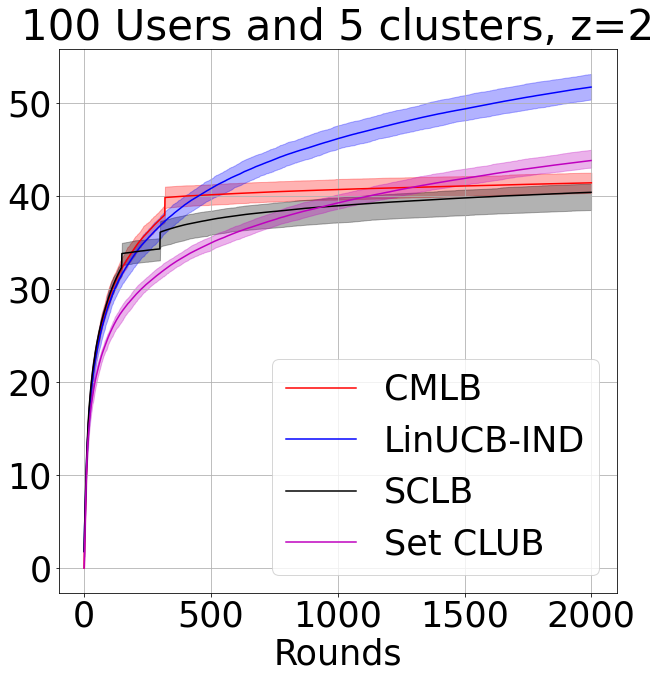}}
    \caption{Synthetic data simulations for clustering.}
    \label{fig:cmlb_simulations}
\end{figure}
\vspace{-2mm}
\subsection{Synthetic Simulations}
\label{sec:synthetic_simulations}
\textbf{Clustering setting :} For each plot of Figures \ref{fig:cmlb_simulations},  users are clustered such that the frequency of cluster $i$ is proportional to $i^{-z}$ (identical to that done in \citep{clustering_online}), where $z$ is mentioned in the figures. Thus for $z=0$, all clusters are balanced, and for larger $z$, the clusters become imbalanced. For each cluster, the unknown parameter vector $\theta^*$ is chosen uniformly at random from the unit sphere. We compare SCLB (ALgorithm~\ref{algo:size_unknown}), CMLB (Algorithm \ref{alg:CMLB}) with  CLUB \citep{clustering_online}, Set CLUB \citep{set_club} and {\ttfamily LinUCB-Ind} which is the simple baseline of no collaboration, where every agent has an independent copy of OFUL. The details of the setup and hyper-parameters are in  Appendix~\ref{app:exp}.  We observe that our algorithm is competitive with respect to CLUB and Set CLUB, and is superior compared to the baseline where each agent is playing an independent copy of OFUL. In particular, we observe either as the clusters become more imbalanced, or as the number of users increases, SCLB and CMLB have a superior performance compared to CLUB and Set CLUB. Furthermore, since SCLB only clusters users logarithmically many number of times, its run-time is faster compared to CLUB.

\textbf{Personalization setting}:  In Figure \ref{fig:pmlb_simulations}, for each plot, we consider a system where the $N$ ground-truth $\theta^*$ vectors are sampled independently from $\mathcal{N}(\mu, \sigma \mathbb{I})$, the normal distribution in $d$ dimensions with mean $\mu \in \mathbb{R}^d$, and variance $\sigma$. The parameter $\mu$ was chosen from the standard normal distirbution in each experiement. We test performance for different values of $\sigma$. Observe that for small $\sigma$, all the ground-truth vectors will be close-by (high structure) and when $\sigma$ is large, the ground-truth vectors are more spread out. We observe in Figure \ref{fig:pmlb_simulations} that PMLB adapts to the available structure. When $\sigma$ is low, in which case every user is close to the average, the regret of PMLB is much lower compared to the baselines. On the other hand, when $\sigma$ is large, i.e., there is no structure to exploit, the regret of PMLB is comparable to the baselines. This demonstrates empirically that PMLB adapts to the problem structure and exploits it whenever present. 




\vspace{-2mm}
\subsection{Comparison on Last.FM Dataset}
\label{sec:real_data}

We compare the CLUB and CMLB on the Last.FM, a real dataset amenable to clustering setting \citep{clustering_online}. LastFM is a collection of $1892$ users and $17632$ artists. This dataset contains records of {\ttfamily (user, artist, tags)} denoting that a user listened to an artist and assigned a tag. More details of the pre-processing, setup and hyper-parameters are in the Appendix. For the two algorithms we plot the ratio of cumulative regret to that obtained by recommending an artist at random each time in Figure~\ref{fig:real_data}. We see that CMLB is competitive. However, the sparsity, renders the task challenging and our results indicate that neither algorithms perform well on this dataset. Note that we observe similar behavior of SCLB as well.



%





\begin{figure}[t!]
\label{fig:synthetic}
\centering
\subfloat[][]{
    \includegraphics[width=0.33\linewidth]{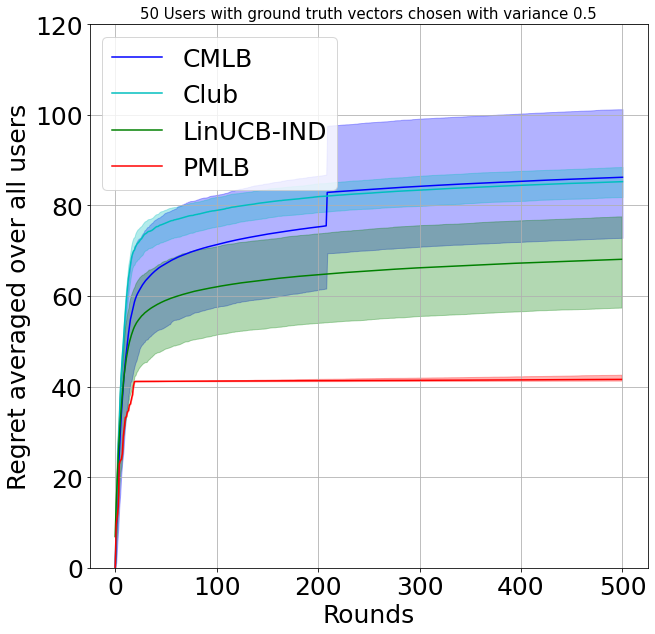}}
    \subfloat[][]{
    \includegraphics[width=0.33\linewidth]{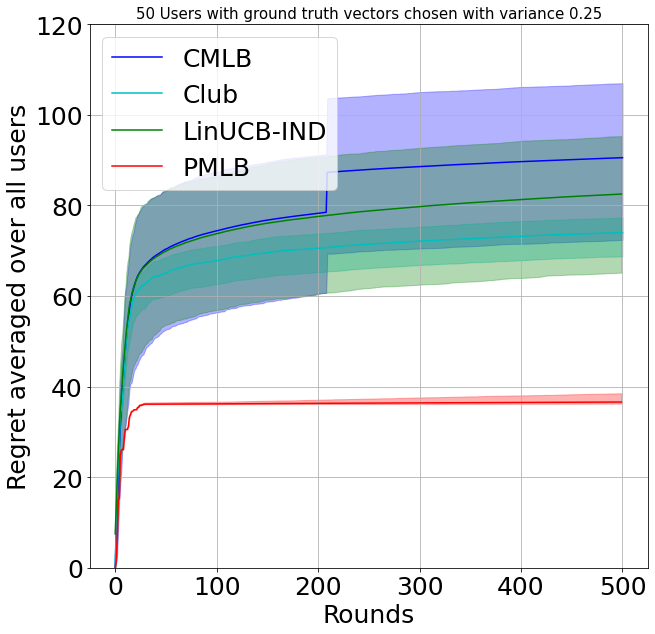}}
    \subfloat[][]{
    \includegraphics[width=0.33\linewidth]{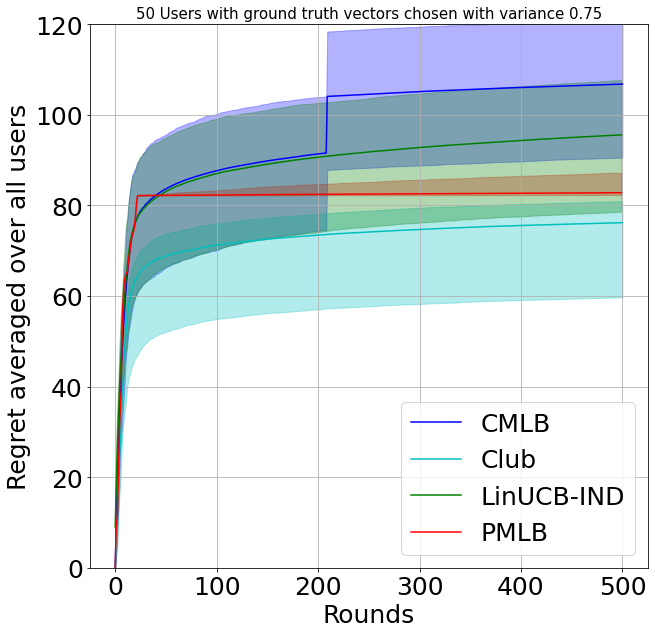}}
    \caption{Synthetic simulations of PMLB.}
    \label{fig:pmlb_simulations}
\end{figure}

\begin{figure}[htb!]
    \centering
    \includegraphics[width=0.35\linewidth]{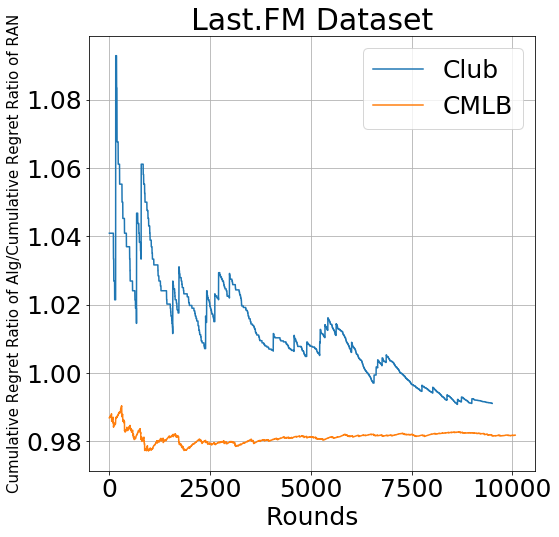}
    \caption{Perf on Last.FM data. Details in Sec \ref{sec:real_data}}
    \label{fig:real_data}
\end{figure}

\section{CONCLUSION}
\vspace{-1mm}
We consider the problem of leveraging user heterogeneity  in a multi-agent stochastic bandit problem under {\em (i)} a clustering and, {\em (ii)} a personalization framework. In both cases, we give novel adaptive algorithms that, without any knowledge of the underlying instance, provides regret guarantees that are sub-linear in $T$ and $N$. A natural avenue for future work will be to combine the two frameworks, where users are all not necessarily identical, but at the same time, their preferences are spread out in space (for example the preference vectors are sampled from a Gaussian mixture model). Natural algorithms here will involve first performing a  clustering on the population, followed by algorithms such as PMLB. Characterizing performance and demonstrating adaptivity in such settings is left to future work. 

\vspace{5mm}

\bibliographystyle{abbrvnat}
\bibliography{clustered_bandits}

\begin{thebibliography}{48}
\providecommand{\natexlab}[1]{#1}
\providecommand{\url}[1]{\texttt{#1}}
\expandafter\ifx\csname urlstyle\endcsname\relax
  \providecommand{\doi}[1]{doi: #1}\else
  \providecommand{\doi}{doi: \begingroup \urlstyle{rm}\Url}\fi

\bibitem[Abbasi-yadkori et~al.(2011)Abbasi-yadkori, P\'{a}l, and
  Szepesv\'{a}ri]{oful}
Y.~Abbasi-yadkori, D.~P\'{a}l, and C.~Szepesv\'{a}ri.
\newblock Improved algorithms for linear stochastic bandits.
\newblock In J.~Shawe-Taylor, R.~Zemel, P.~Bartlett, F.~Pereira, and K.~Q.
  Weinberger, editors, \emph{Advances in Neural Information Processing
  Systems}, volume~24, pages 2312--2320. Curran Associates, Inc., 2011.

\bibitem[Arora et~al.(2020)Arora, Du, Kakade, Luo, and
  Saunshi]{arora2020provable}
S.~Arora, S.~Du, S.~Kakade, Y.~Luo, and N.~Saunshi.
\newblock Provable representation learning for imitation learning via bi-level
  optimization.
\newblock In \emph{International Conference on Machine Learning}, pages
  367--376. PMLR, 2020.

\bibitem[Balakrishnan et~al.(2017)Balakrishnan, Wainwright, Yu, et~al.]{em_2}
S.~Balakrishnan, M.~J. Wainwright, B.~Yu, et~al.
\newblock Statistical guarantees for the em algorithm: From population to
  sample-based analysis.
\newblock \emph{Annals of Statistics}, 45\penalty0 (1):\penalty0 77--120, 2017.

\bibitem[Cesa-Bianchi et~al.(2013)Cesa-Bianchi, Gentile, and
  Zappella]{gang_bandits}
N.~Cesa-Bianchi, C.~Gentile, and G.~Zappella.
\newblock A gang of bandits.
\newblock \emph{arXiv preprint arXiv:1306.0811}, 2013.

\bibitem[Chatterji et~al.(2020)Chatterji, Muthukumar, and
  Bartlett]{chatterji2020osom}
N.~Chatterji, V.~Muthukumar, and P.~Bartlett.
\newblock Osom: A simultaneously optimal algorithm for multi-armed and linear
  contextual bandits.
\newblock In \emph{International Conference on Artificial Intelligence and
  Statistics}, pages 1844--1854. PMLR, 2020.

\bibitem[Chu et~al.(2011)Chu, Li, Reyzin, and Schapire]{chu2011contextual}
W.~Chu, L.~Li, L.~Reyzin, and R.~Schapire.
\newblock Contextual bandits with linear payoff functions.
\newblock In \emph{Proceedings of the Fourteenth International Conference on
  Artificial Intelligence and Statistics}, pages 208--214. JMLR Workshop and
  Conference Proceedings, 2011.

\bibitem[Collins et~al.(2021)Collins, Hassani, Mokhtari, and
  Shakkottai]{collins2021exploiting}
L.~Collins, H.~Hassani, A.~Mokhtari, and S.~Shakkottai.
\newblock Exploiting shared representations for personalized federated
  learning.
\newblock \emph{arXiv preprint arXiv:2102.07078}, 2021.

\bibitem[Covington et~al.(2016)Covington, Adams, and Sargin]{covington2016deep}
P.~Covington, J.~Adams, and E.~Sargin.
\newblock Deep neural networks for youtube recommendations.
\newblock In \emph{Proceedings of the 10th ACM conference on recommender
  systems}, pages 191--198, 2016.

\bibitem[Denevi et~al.(2019)Denevi, Ciliberto, Grazzi, and Pontil]{meta}
G.~Denevi, C.~Ciliberto, R.~Grazzi, and M.~Pontil.
\newblock Learning-to-learn stochastic gradient descent with biased
  regularization.
\newblock In K.~Chaudhuri and R.~Salakhutdinov, editors, \emph{Proceedings of
  the 36th International Conference on Machine Learning}, volume~97 of
  \emph{Proceedings of Machine Learning Research}, pages 1566--1575. PMLR,
  09--15 Jun 2019.
\newblock URL \url{http://proceedings.mlr.press/v97/denevi19a.html}.

\bibitem[D'Eramo et~al.(2019)D'Eramo, Tateo, Bonarini, Restelli, and
  Peters]{d2019sharing}
C.~D'Eramo, D.~Tateo, A.~Bonarini, M.~Restelli, and J.~Peters.
\newblock Sharing knowledge in multi-task deep reinforcement learning.
\newblock In \emph{International Conference on Learning Representations}, 2019.

\bibitem[Fallah et~al.(2020{\natexlab{a}})Fallah, Mokhtari, and
  Ozdaglar]{NEURIPS2020_24389bfe}
A.~Fallah, A.~Mokhtari, and A.~Ozdaglar.
\newblock Personalized federated learning with theoretical guarantees: A
  model-agnostic meta-learning approach.
\newblock In H.~Larochelle, M.~Ranzato, R.~Hadsell, M.~F. Balcan, and H.~Lin,
  editors, \emph{Advances in Neural Information Processing Systems}, volume~33,
  pages 3557--3568. Curran Associates, Inc., 2020{\natexlab{a}}.
\newblock URL
  \url{https://proceedings.neurips.cc/paper/2020/file/24389bfe4fe2eba8bf9aa9203a44cdad-Paper.pdf}.

\bibitem[Fallah et~al.(2020{\natexlab{b}})Fallah, Mokhtari, and
  Ozdaglar]{fallah2020convergence}
A.~Fallah, A.~Mokhtari, and A.~Ozdaglar.
\newblock On the convergence theory of gradient-based model-agnostic
  meta-learning algorithms.
\newblock In \emph{International Conference on Artificial Intelligence and
  Statistics}, pages 1082--1092. PMLR, 2020{\natexlab{b}}.

\bibitem[Fallah et~al.(2020{\natexlab{c}})Fallah, Mokhtari, and
  Ozdaglar]{fallah2020personalized}
A.~Fallah, A.~Mokhtari, and A.~Ozdaglar.
\newblock Personalized federated learning: A meta-learning approach.
\newblock \emph{arXiv preprint arXiv:2002.07948}, 2020{\natexlab{c}}.

\bibitem[Finn et~al.(2019)Finn, Rajeswaran, Kakade, and Levine]{finn2019online}
C.~Finn, A.~Rajeswaran, S.~Kakade, and S.~Levine.
\newblock Online meta-learning.
\newblock In \emph{International Conference on Machine Learning}, pages
  1920--1930. PMLR, 2019.

\bibitem[Foster et~al.(2019)Foster, Krishnamurthy, and Luo]{foster2019model}
D.~J. Foster, A.~Krishnamurthy, and H.~Luo.
\newblock Model selection for contextual bandits, 2019.

\bibitem[Gentile et~al.(2014)Gentile, Li, and Zappella]{clustering_online}
C.~Gentile, S.~Li, and G.~Zappella.
\newblock Online clustering of bandits.
\newblock In \emph{International Conference on Machine Learning}, pages
  757--765. PMLR, 2014.

\bibitem[Gentile et~al.(2017)Gentile, Li, Kar, Karatzoglou, Zappella, and
  Etrue]{gentile2017context}
C.~Gentile, S.~Li, P.~Kar, A.~Karatzoglou, G.~Zappella, and E.~Etrue.
\newblock On context-dependent clustering of bandits.
\newblock In \emph{International Conference on Machine Learning}, pages
  1253--1262. PMLR, 2017.

\bibitem[Ghosh et~al.(2021{\natexlab{a}})Ghosh, Chowdhury, and
  Ramchandran]{ghosh2021model}
A.~Ghosh, S.~R. Chowdhury, and K.~Ramchandran.
\newblock Model selection with near optimal rates for reinforcement learning
  with general model classes.
\newblock \emph{arXiv preprint arXiv:2107.05849}, 2021{\natexlab{a}}.

\bibitem[Ghosh et~al.(2021{\natexlab{b}})Ghosh, Sankararaman, and
  Kannan]{ghosh_adaptive}
A.~Ghosh, A.~Sankararaman, and R.~Kannan.
\newblock Problem-complexity adaptive model selection for stochastic linear
  bandits.
\newblock In A.~Banerjee and K.~Fukumizu, editors, \emph{Proceedings of The
  24th International Conference on Artificial Intelligence and Statistics},
  volume 130 of \emph{Proceedings of Machine Learning Research}, pages
  1396--1404. PMLR, 13--15 Apr 2021{\natexlab{b}}.
\newblock URL \url{http://proceedings.mlr.press/v130/ghosh21a.html}.

\bibitem[Higgins et~al.(2017)Higgins, Pal, Rusu, Matthey, Burgess, Pritzel,
  Botvinick, Blundell, and Lerchner]{higgins2017darla}
I.~Higgins, A.~Pal, A.~Rusu, L.~Matthey, C.~Burgess, A.~Pritzel, M.~Botvinick,
  C.~Blundell, and A.~Lerchner.
\newblock Darla: Improving zero-shot transfer in reinforcement learning.
\newblock In \emph{International Conference on Machine Learning}, pages
  1480--1490. PMLR, 2017.

\bibitem[Khodak et~al.(2019)Khodak, Balcan, and Talwalkar]{khodak2019adaptive}
M.~Khodak, M.-F. Balcan, and A.~Talwalkar.
\newblock Adaptive gradient-based meta-learning methods.
\newblock \emph{arXiv preprint arXiv:1906.02717}, 2019.

\bibitem[Korda et~al.(2016)Korda, Szorenyi, and Li]{korda2016distributed}
N.~Korda, B.~Szorenyi, and S.~Li.
\newblock Distributed clustering of linear bandits in peer to peer networks.
\newblock In \emph{International conference on machine learning}, pages
  1301--1309. PMLR, 2016.

\bibitem[Kwon and Caramanis(2020)]{em_1}
J.~Kwon and C.~Caramanis.
\newblock The em algorithm gives sample-optimality for learning mixtures of
  well-separated gaussians.
\newblock In \emph{Conference on Learning Theory}, pages 2425--2487. PMLR,
  2020.

\bibitem[Lazaric and Restelli(2011)]{lazaric}
A.~Lazaric and M.~Restelli.
\newblock Transfer from multiple mdps.
\newblock In J.~Shawe-Taylor, R.~Zemel, P.~Bartlett, F.~Pereira, and K.~Q.
  Weinberger, editors, \emph{Advances in Neural Information Processing
  Systems}, volume~24. Curran Associates, Inc., 2011.
\newblock URL
  \url{https://proceedings.neurips.cc/paper/2011/file/fe7ee8fc1959cc7214fa21c4840dff0a-Paper.pdf}.

\bibitem[Lee(2001)]{rec_sys_collab_learning}
W.~S. Lee.
\newblock Collaborative learning for recommender systems.
\newblock In \emph{ICML}, volume~1, pages 314--321. Citeseer, 2001.

\bibitem[Li et~al.(2010)Li, Chu, Langford, and Schapire]{langford_news}
L.~Li, W.~Chu, J.~Langford, and R.~E. Schapire.
\newblock A contextual-bandit approach to personalized news article
  recommendation.
\newblock In \emph{Proceedings of the 19th international conference on World
  wide web}, pages 661--670, 2010.

\bibitem[Li and Kim(2003)]{cluster_rec}
Q.~Li and B.~M. Kim.
\newblock Clustering approach for hybrid recommender system.
\newblock In \emph{Proceedings IEEE/WIC International Conference on Web
  Intelligence (WI 2003)}, pages 33--38. IEEE, 2003.

\bibitem[Li et~al.(2019)Li, Chen, and Leung]{set_club}
S.~Li, W.~Chen, and K.-S. Leung.
\newblock Improved algorithm on online clustering of bandits.
\newblock \emph{arXiv preprint arXiv:1902.09162}, 2019.

\bibitem[Li et~al.(2020)Li, Hu, Beirami, and Smith]{ditto-virginia}
T.~Li, S.~Hu, A.~Beirami, and V.~Smith.
\newblock Federated multi-task learning for competing constraints.
\newblock \emph{CoRR}, abs/2012.04221, 2020.
\newblock URL \url{https://arxiv.org/abs/2012.04221}.

\bibitem[Linden et~al.(2003)Linden, Smith, and York]{amazon}
G.~Linden, B.~Smith, and J.~York.
\newblock Amazon. com recommendations: Item-to-item collaborative filtering.
\newblock \emph{IEEE Internet computing}, 7\penalty0 (1):\penalty0 76--80,
  2003.

\bibitem[Liu et~al.(2015)Liu, Liu, Chua, and Sun]{word_embeddings_1}
Y.~Liu, Z.~Liu, T.-S. Chua, and M.~Sun.
\newblock Topical word embeddings.
\newblock In \emph{Proceedings of the AAAI Conference on Artificial
  Intelligence}, volume~29, 2015.

\bibitem[Ma et~al.(2020)Ma, Narayanaswamy, Lin, and Ding]{murali}
Y.~Ma, B.~Narayanaswamy, H.~Lin, and H.~Ding.
\newblock Temporal-contextual recommendation in real-time.
\newblock In \emph{Proceedings of the 26th ACM SIGKDD International Conference
  on Knowledge Discovery \& Data Mining}, pages 2291--2299, 2020.

\bibitem[Mansour et~al.(2020)Mansour, Mohri, Ro, and Suresh]{three-google}
Y.~Mansour, M.~Mohri, J.~Ro, and A.~T. Suresh.
\newblock Three approaches for personalization with applications to federated
  learning.
\newblock \emph{CoRR}, abs/2002.10619, 2020.
\newblock URL \url{https://arxiv.org/abs/2002.10619}.

\bibitem[Naumov et~al.(2019)Naumov, Mudigere, Shi, Huang, Sundaraman, Park,
  Wang, Gupta, Wu, Azzolini, et~al.]{rec_deep}
M.~Naumov, D.~Mudigere, H.-J.~M. Shi, J.~Huang, N.~Sundaraman, J.~Park,
  X.~Wang, U.~Gupta, C.-J. Wu, A.~G. Azzolini, et~al.
\newblock Deep learning recommendation model for personalization and
  recommendation systems.
\newblock \emph{arXiv preprint arXiv:1906.00091}, 2019.

\bibitem[Okura et~al.(2017)Okura, Tagami, Ono, and Tajima]{deep_embedding}
S.~Okura, Y.~Tagami, S.~Ono, and A.~Tajima.
\newblock Embedding-based news recommendation for millions of users.
\newblock In \emph{Proceedings of the 23rd ACM SIGKDD International Conference
  on Knowledge Discovery and Data Mining}, pages 1933--1942, 2017.

\bibitem[Ozsoy(2016)]{embedding_1}
M.~G. Ozsoy.
\newblock From word embeddings to item recommendation.
\newblock \emph{arXiv preprint arXiv:1601.01356}, 2016.

\bibitem[Pal et~al.(2020)Pal, Eksombatchai, Zhou, Zhao, Rosenberg, and
  Leskovec]{pintrest}
A.~Pal, C.~Eksombatchai, Y.~Zhou, B.~Zhao, C.~Rosenberg, and J.~Leskovec.
\newblock Pinnersage: Multi-modal user embedding framework for recommendations
  at pinterest.
\newblock In \emph{Proceedings of the 26th ACM SIGKDD International Conference
  on Knowledge Discovery \& Data Mining}, pages 2311--2320, 2020.

\bibitem[Parisotto et~al.(2015)Parisotto, Ba, and
  Salakhutdinov]{parisotto2015actor}
E.~Parisotto, J.~L. Ba, and R.~Salakhutdinov.
\newblock Actor-mimic: Deep multitask and transfer reinforcement learning.
\newblock \emph{arXiv preprint arXiv:1511.06342}, 2015.

\bibitem[Rusu et~al.(2015)Rusu, Colmenarejo, Gulcehre, Desjardins, Kirkpatrick,
  Pascanu, Mnih, Kavukcuoglu, and Hadsell]{rusu2015policy}
A.~A. Rusu, S.~G. Colmenarejo, C.~Gulcehre, G.~Desjardins, J.~Kirkpatrick,
  R.~Pascanu, V.~Mnih, K.~Kavukcuoglu, and R.~Hadsell.
\newblock Policy distillation.
\newblock \emph{arXiv preprint arXiv:1511.06295}, 2015.

\bibitem[Sarwar et~al.(2002)Sarwar, Karypis, Konstan, and
  Riedl]{cluster_rec_2r}
B.~M. Sarwar, G.~Karypis, J.~Konstan, and J.~Riedl.
\newblock Recommender systems for large-scale e-commerce: Scalable neighborhood
  formation using clustering.
\newblock In \emph{Proceedings of the fifth international conference on
  computer and information technology}, volume~1, pages 291--324. Citeseer,
  2002.

\bibitem[Saveski and Mantrach(2014)]{cold_start_embedding}
M.~Saveski and A.~Mantrach.
\newblock Item cold-start recommendations: learning local collective
  embeddings.
\newblock In \emph{Proceedings of the 8th ACM Conference on Recommender
  systems}, pages 89--96, 2014.

\bibitem[Vershynin(2011)]{vershynin2011introduction}
R.~Vershynin.
\newblock Introduction to the non-asymptotic analysis of random matrices, 2011.

\bibitem[Wang et~al.(2016)Wang, Tang, Aggarwal, and Liu]{doc_embedding}
S.~Wang, J.~Tang, C.~Aggarwal, and H.~Liu.
\newblock Linked document embedding for classification.
\newblock In \emph{Proceedings of the 25th ACM international on conference on
  information and knowledge management}, pages 115--124, 2016.

\bibitem[Xue et~al.(2017)Xue, Dai, Zhang, Huang, and
  Chen]{deep_matrix_factor_1}
H.-J. Xue, X.~Dai, J.~Zhang, S.~Huang, and J.~Chen.
\newblock Deep matrix factorization models for recommender systems.
\newblock In \emph{IJCAI}, volume~17, pages 3203--3209. Melbourne, Australia,
  2017.

\bibitem[Yang et~al.(2021)Yang, Hu, Lee, and Du]{yang2021impact}
J.~Yang, W.~Hu, J.~D. Lee, and S.~S. Du.
\newblock Impact of representation learning in linear bandits.
\newblock In \emph{International Conference on Learning Representations}, 2021.
\newblock URL \url{https://openreview.net/forum?id=edJ_HipawCa}.

\bibitem[Yao et~al.(2020)Yao, Yi, Cheng, Yu, Menon, Hong, Chi, Tjoa, Ettinger,
  et~al.]{deep_learn_rec}
T.~Yao, X.~Yi, D.~Z. Cheng, F.~Yu, A.~Menon, L.~Hong, E.~H. Chi, S.~Tjoa,
  E.~Ettinger, et~al.
\newblock Self-supervised learning for deep models in recommendations.
\newblock \emph{arXiv preprint arXiv:2007.12865}, 2020.

\bibitem[Zhao et~al.(2017)Zhao, Ding, and Fu]{deep_matrix_factor_2}
H.~Zhao, Z.~Ding, and Y.~Fu.
\newblock Multi-view clustering via deep matrix factorization.
\newblock In \emph{Proceedings of the AAAI Conference on Artificial
  Intelligence}, volume~31, 2017.

\bibitem[Zhao et~al.(2019)Zhao, Hong, Wei, Chen, Nath, Andrews, Kumthekar,
  Sathiamoorthy, Yi, and Chi]{zhao2019recommending}
Z.~Zhao, L.~Hong, L.~Wei, J.~Chen, A.~Nath, S.~Andrews, A.~Kumthekar,
  M.~Sathiamoorthy, X.~Yi, and E.~Chi.
\newblock Recommending what video to watch next: a multitask ranking system.
\newblock In \emph{Proceedings of the 13th ACM Conference on Recommender
  Systems}, pages 43--51, 2019.

\end{thebibliography}

\newpage
\begin{center}
    \textbf{\Large{Supplementary Material for ``Adaptive Clustering and Personalization in  Multi-Agent Stochastic Linear Bandits''}}
\end{center}
\vspace{4mm}
\section{Additional Details on Simulations}
\label{app:exp}

\subsection{Synthetic Data}

\textbf{Setup:} In each setting, we simulate all algorithms with the $25$ context-vectors, each of dimension $15$, sampled at random from  $\mathsf{Unif}[-1/\sqrt{d},1/\sqrt{d}]^{\otimes d}$. The plots in Figures~\ref{fig:cmlb_simulations} and \ref{fig:pmlb_simulations} show the regret averaged over all users, after each algorithm has taken $1000$ steps for all users with 30 repetetions. SCLB, CMLB and LinUCB-Ind take a total of $1000$ rounds, while CLUB takes $1000 \times \text{{\ttfamily num-agents}}$ rounds. For CLUB, users are picked in a round-robin fashion, with all users shown the same set of contexts in a batch. Thus at the $t$-th arm-pull in all algorithms, all users have the same set of contexts. We repeat 30 times and plot $95$-th percentile confidence bounds of the regret averaged over users.

\textbf{Hyper-parameters :} For CMLB, in all experiments, we use $\delta = 0.4$, $\alpha = 0.2$, $C = 0.2$ and $p^{*}=0$. For LinUCB, we use $\lambda=1$. For CLUB, we tuned the two hyper-parameters $\alpha$ and $\alpha_2$ for each setting, by considering the performance over the first 500 rounds and choosing the best one.

\subsection{Real Data}

\textbf{Data:} LastFM is a collection of $1892$ users and $17632$ artists. This dataset contains records of {\ttfamily (user, artist, tags)} denoting that a user listened to an artist and assigned a tag.
 We convert this into a multi-agent recommendation task, identical to the setting considered in \citep{clustering_online} and \citep{gang_bandits}. 
We break down all tags into atomic units, exactly as suggested in \citep{clustering_online, gang_bandits}, and assign to every artist, the collection of assigned atomic tags by all users. We then extract the top $25$ principal components from the \emph{tf-idf} matrix of artists and atomic tags as the context vectors for artists. Thus, each artist is a $25$ dimensional vector. The reward for a {\ttfamily (user,artist)} pair is $1$ if present in the dataset; else $0$.

\textbf{Setup:} We consider a time-horizon of $10000$ - CMLB was simulated for $10000$ rounds and CLUB until all users had taken $10000$ steps. At a given time instant, a set of $25$ randomly sampled items was shown as contexts for CMLB. For CLUB, we first chose a user by picking them in a round robin fashion, and choose $24$ items at random and one item at random from among those the user had listened to. This way, we ensure that at each time, the best reward for CLUB is at-least one. However, for CMLB, as at each time step all users play, such a guarantee cannot be made. This makes the learning setting harder for CMLB since at every time, a large fraction of users have the best-reward of $0$, i.e., the arm separation is $0$, while the best reward is always $1$ for CLUB.

\textbf{Hyper-parameters:} We use $\delta=0.3$, $\alpha = C = 0.5$ for CMLB. For CLUB, we use $\alpha = 1$ and $\alpha_2 = 2$ chosen from a burn-in period of $500$ arm-pulls of all agents.

\textbf{Results:} We compare the two algorithms by plotting the ratio of cumulative regret to that obtained by recommending an artist at random each time in Figure~\ref{fig:real_data}. We see that CMLB is competitive. However, the sparsity, renders the task quite challenging and our results indicate that neither algorithms are particularly appealing for this dataset.

\section{{\ttfamily ALB-Norm} from \citep{ghosh_adaptive}}
In this section, we reproduce {\ttfamily ALB-Norm} from \citep{ghosh_adaptive}, and prove a Corollary of the main theorem from \citep{ghosh_adaptive}.

\begin{algorithm}[t!]
  \caption{Adaptive Linear Bandit (norm)--{{\ttfamily ALB-Norm}}}
  \begin{algorithmic}[1]
 \STATE  \textbf{Input:} Initial exploration period $\tau$, the phase length $T_1 := \lceil \sqrt{T} \rceil$,  $\delta_1 > 0$, $\delta_s > 0$.
 \STATE Select an arm at random, sample $2\tau$ rewards
 \STATE Obtain initial estimate ($b_1$) of $\|\theta^*\|$ according to Section $3.3$ of \citep{ghosh_adaptive}.
 \FOR{$t = 1,2,\ldots,K$}
 \STATE Play arm $t$, receive reward $g_{t,t}$
 \ENDFOR
 \STATE Define $\mathcal{S} = \{g_{i,i}\}_{i=1}^K$ 
  \FOR{ epochs $i=1,2 \ldots, N $}
  \STATE Use $\mathcal{S}$ as pure-exploration reward
  \STATE Play OFUL$_{\delta_i}^+(b_{i})$ until the end of epoch $i$ (denoted by $\mathcal{E}_i$)
  \STATE At $t=\mathcal{E}_i$, refine estimate of $\|\theta^*\|$ as, $ b_{i+1} = \max_{\theta \in \mathcal{C}_{\mathcal{E}_i}} \|\theta\|$
  \STATE Set $T_{i+1} = 2 T_{i}$, $\delta_{i+1} = \frac{\delta_i}{2}$.
    \ENDFOR
    \STATE \underline{\texttt{OFUL$^+_\delta(b)$:}}
     \STATE  \textbf{Input:} Parameters $b$, $\delta >0$, number of rounds $\Tilde{T}$
 \FOR{$t = 1,2, \ldots, \Tilde{T} $}
 \STATE Select the best arm estimate as $
     j_t = \mathrm{argmax}_{i\in [K]} \left[ \max_{\theta \in \mathcal{C}_{t-1}} \{  \langle \alpha_{i,t}, \theta \rangle \} \right]$, \\
     where $\mathcal{C}_t$ is given in Section $3.2$ of \citep{ghosh_adaptive}.
 \STATE Play arm $j_t$, and update $\mathcal{C}_{t}$
 \ENDFOR
  \end{algorithmic}
  \label{algo:alb_norm}
\end{algorithm}

\begin{corollary}[Corollary of Theorem $1$ from \citep{ghosh_adaptive}]
The regret of Algorithm \ref{algo:alb_norm} at the end of $T$ time-steps satisfies with probability at-least $1- 18\delta_1 - \delta_s$, 
\begin{align*}
    R(T) \leq C\|\theta^*\|(\sqrt{K}+\sqrt{d})\sqrt{T}\log\left(\frac{KT}{\delta_1}\right),
\end{align*}
where $C$ is an universal constant.
\label{cor:alb_norm_improved}
\end{corollary}

The proof follows by recomputing Lemma $1$ from \citep{ghosh_adaptive} as follows.
\begin{lemma}
If $T$ is sufficiently large such that $\frac{2C\sigma\sqrt{d}}{T^{\frac{1}{4}}}\log \left( \frac{K \sqrt{T}}{\delta_1}\right) \leq 1$, then with probability at-least $1-8\delta_1 - \delta_s$, for all $i$ large, $b_i \leq 2 \|\theta^* \|$ holds, where $b_i$ is defined in Line $11$ of Algorithm \ref{algo:alb_norm}.
\label{lem:recompute}
\end{lemma}
\begin{proof}[Proof of Lemma \ref{lem:recompute}]
We start with Equation $(8)$ of \citep{ghosh_adaptive}. Reproducing Equation $(8)$ by substituting $T_1 = \lceil \sqrt{T} \rceil$, with probability at-least $1-8\delta_1$, for all phases $i \geq 2$,
\begin{align}
   b_{i+1} \leq \|\theta^*\| + ip\frac{b_i}{2^{\frac{i-1}{2}}T^{\frac{1}{4}}} + iq\frac{\sqrt{d}}{2^{\frac{i-1}{2}}T^{\frac{1}{4}}},
   \label{eqn:base_recursion}
\end{align}
holds, where $p$ and $q$ are defined in \citep{ghosh_adaptive} as 
\begin{align*}
    p &= \left( \frac{14\log \left( \frac{2K\sqrt{T}}{\delta_1} \right)}{\sqrt{\rho_{min}}} \right),\\
    q &= \left( \frac{2C\sigma \log \left( \frac{2K\sqrt{T}}{\delta_1} \right)}{\sqrt{\rho_{min}}} \right).
\end{align*}

For all $i \geq 2$, $\frac{i}{2^{\frac{i-1}{2}}} \leq 2$. Thus, for all $i \geq 1$, Equation (\ref{eqn:base_recursion}) can be rewritten as 
\begin{align}
    b_{i+1} &\leq \|\theta^* \| + \frac{pb_i}{T^{\frac{1}{4}}} + \frac{q \sqrt{d}}{T^{\frac{1}{4}}},\nonumber \\
    &\leq \|\theta^*\| + \frac{C\sigma\sqrt{d}}{T^{\frac{1}{4}}}\log \left( \frac{K \sqrt{T}}{\delta_1}\right)b_i.
       \label{eqn:base_recursion_2}
\end{align}
where $b_1 := 1$. We set this initial estimate as $1$, since $\max_{i \in \{1,\cdots, N\}}\|\theta^*_i\| \leq 1$. We prove the lemma by induction that $b_i \leq 2\|\theta^*\|$. 
\\

\textbf{Base case, $i=1$} - We know from the initialization (Line $3$ of Algorithm \ref{algo:alb_norm}), that with probability at-least $1-\delta_s$,
\begin{align*}
    b_1 &\leq \|\theta^*\| + \sqrt{2}\sigma \sqrt{\frac{d}{\tau}\log \left( \frac{1}{\delta_s}\right)}, \\
    &\leq 2 \|\theta^*\|.
\end{align*}
where $\tau$ and $\delta_s$ are defined in Line $2$ and input respectively of Algorithm \ref{algo:alb_norm}. 
\\

\textbf{Induction Step} - Assume that for some $i \geq 1$, for all $1 \leq j \leq i$, $b_j \leq 2 \|\theta^*\|$. Now, consider case $i+1$. From recursion in Equation (\ref{eqn:base_recursion_2}), that
\begin{align*}
    b_{i+1} &\leq \|\theta^* \| +\frac{C\sigma\sqrt{d}}{T^{\frac{1}{4}}}\log \left( \frac{K \sqrt{T}}{\delta_1}\right)b_i,\\
    &\stackrel{(a)}{\leq} \|\theta^*\|\left( 1 + \frac{2C\sigma\sqrt{d}}{T^{\frac{1}{4}}}\log \left( \frac{K \sqrt{T}}{\delta_1}\right) \right), \\
    &\stackrel{(b)}{\leq} 2 \|\theta^*\|.
\end{align*}
Step $(a)$ follows from the induction hypothesis. Step $(b)$ follows from the fact that $T$ is large enough such that $\frac{2C\sigma\sqrt{d}}{T^{\frac{1}{4}}}\log \left( \frac{K \sqrt{T}}{\delta_1}\right) \leq 1$. This concludes the proof of Lemma.
\end{proof}

\section{Proof of Lemma~\ref{lem:separated_cluster}}
Here, there is a gap between the optimal parameters. In this case, suppose the Individual Learning phase lasts for $C^{(2)} \frac{d (NT)^{2\alpha}}{\rho_{\min}} \log(1/\delta)$ time steps. Following the analysis of OFUL \citep{oful,chatterji2020osom}, alonf=g with the condition in equation~\ref{eqn:tau}, after $t$ instances, we have
\begin{align*}
    \|\Hat{\theta}_t^{(i)} -\theta^*_1\| \leq \frac{D}{\sqrt{1+\rho_{\min}t/2}},
\end{align*}
with probability at least $1-\delta$, where $D = \Tilde{\mathcal{O}}(\sqrt{d})\log(1/\delta)$.

If agents $i$ and $j$ fall in same cluster, we have,
\begin{align*}
    \|\Hat{\theta}^{(j)} - \Hat{\theta}^{(i)}\| \leq 2/(NT)^\alpha.
\end{align*}
with probability at least $1-2\delta$.

Otherwise we have
\begin{align*}
    \|\Hat{\theta}^{(j)} - \Hat{\theta}^{(i)}\| \geq \Delta_i - 2/(NT)^\alpha.
    \end{align*}
Now, suppose $3/(NT)^\alpha \leq \Delta_i - 2/(NT)^\alpha$, or in other words, $\Delta_i \geq 5/(NT)^\alpha$. In that case,
\begin{align*}
    \|\Hat{\theta}^{(j)} - \Hat{\theta}^{(i)}\| \geq  3/(NT)^\alpha.
    \end{align*}
    with high probability. So, if we threshold $ 3/(NT)^\alpha$, we can find out the cluster perfectly with probability exceeding $1-2\delta$. Since we want this to hold for every pair of agents, a simple union bound yields the lemma.
    
Since, there is no clustering error, we have
\begin{align*}
    R(T) = R(ind-learn) + R(coll-learn)
\end{align*}
The Individual Learning phase continues until  $C^{(2)} \frac{d (NT)^{2\alpha}}{\rho_{\min}} \log(1/\delta)$ time steps. Hence, according to \citep{chatterji2020osom}, we have
\begin{align*}
    R(ind-learn) &\leq C \sqrt{\frac{d}{\rho_{\min}}} \left(\sqrt{C^{(2)} \frac{d (NT)^{2\alpha}}{\rho_{\min}} \log(1/\delta)} \right) \sqrt{\log(1/\delta)} \\
    &\leq C_1 \,\, \frac{d}{\rho_{\min}}  (NT)^\alpha \,\, \log(1/\delta)
\end{align*}
with probability greater than $1-\delta$. To avoid clutter, we have only considered the leading term in the above regret.

We now characterize the regret in the collaborative learning phase. Here, the regret depends on the cluster size. Since the center averages the mean reward from all the users in a cluster, it effectively reduces the noise variance by a factor of the cluster size. Hence, the regret upper bound is we have (using \citep{chatterji2020osom}),
\begin{align*}
R(coll-learn) \leq C_1 \sqrt{\frac{d}{\rho_{\min}} } \,\, \left(\sqrt{\frac{T-\frac{d (NT)^{2\alpha}}{\rho_{\min}} \log(1/\delta)}{p_i N}} \right) \sqrt{\log(1/\delta)}
\end{align*}
with probability at least $1-\delta$. Since, the size of the $i$-th cluster is $p_i$.

Hence, with probability at least $1-2\delta$ total regret is given by
\begin{align*}
    R(T) &\leq C_1  \left[\frac{d}{\rho_{\min}}  (NT)^\alpha + \sqrt{\frac{d}{\rho_{\min}} } \,\, \left(\sqrt{\frac{T-\frac{d (NT)^{2\alpha}}{\rho_{\min}} \log(1/\delta)}{p_i N}} \right) \right] \log(1/\delta).
     \end{align*}
Suppose $\alpha$ satisfies
\begin{align*}
    \alpha \leq \frac{1}{2} \left( \frac{\log\left[\frac{\rho_{\min}T}{d p_i N}\right]}{\log(NT)}\right).
\end{align*}
Then, the first term in the above regret expression can be upper bounded by the second term, and the resulting regret is given by
\begin{align*}
    R(T) \leq C \left(\sqrt{\frac{d}{\rho_{\min}}} \,\, \sqrt{\frac{T}{p_i N}} \right) \log(1/\delta)
\end{align*}
with probability at least $1-2\delta$.

\section{Proof of Lemma~\ref{lem:inseparable_clusters}}
In this case, we have $\Delta_i \leq 5/(NT)^\alpha$. In this case, we show that the maximal-cluster subroutine of Algorithm~\ref{alg:CMLB} treats the neighboring clusters of cluster $i$, also together with $i$, as a single cluster, with high probability. It may happen that some of the clusters are left out owing to being far from cluster $i$. Let $\mathcal{S}$ be the set of cluster indices that Algorithm~\ref{alg:CMLB} clubs with cluster $i$. It is easy to see $\max_{j \in \mathcal{S}}\|\theta^*_i - \theta^*_j\| \leq 5L/(NT)^\alpha$; otherwise we will be in separable cluster setting.

Note that in this case, we have no high probability guarantees on the cluster assignment by Algorithm~\ref{alg:CMLB}. However, in this situation also, we argue that the regret suffered by the users are not very large. This is because the maximum separation between the clusters (and hence the clustering error) is $\mathcal{O}(L/(NT)^\alpha)$, which is quite small.

Let us now focus on the regret upper-bound. The regret is given by
\begin{align*}
    R(T) = R(ind-learn) + R(coll-learn) 
    + R(cluster-error)
\end{align*}
The first term comes from the initial phase of our algorithm. The second term comes after the (one-phase) clustering, and thereby exploiting the clustered OFUL algorithm. The third term comes when the algorithm makes an error in parameter estimates. Here Algorithm~\ref{alg:CMLB} clubs several clusters with cluster $i$, and hence one needs to address the clustering error. This clustering error indeed accumulates over the rest of the play.

The regret in the individual learning follows  analysis similar to Lemma~\ref{lem:separated_cluster}. We obtain

\begin{align*}
    R(ind-learn) &\leq C \sqrt{\frac{d}{\rho_{\min}}} \left(\sqrt{C^{(2)} \frac{d (NT)^{2\alpha}}{\rho_{\min}} \log(1/\delta)} \right) \sqrt{\log(1/\delta)} \\
    &\leq C_1 \,\, \frac{d}{\rho_{\min}}  (NT)^\alpha \,\, \log(1/\delta)
\end{align*}
with probability greater than $1-\delta$.

Let us now consider the collaborative learning phase. In this case, the maximal-cluster subroutine of CMLB treats the neighboring clusters of cluster $i$, also together with $i$, as a single cluster, with high probability. It may happen that some of the clusters are left out owing to being far from cluster $i$. Let $\mathcal{S}$ be the set of cluster indices that Algorithm~\ref{alg:CMLB} clubs with cluster $i$. It is easy to see $\max_{j \in \mathcal{S}}\|\theta^*_i - \theta^*_j\| \leq 5L/(NT)^\alpha$; otherwise we will be in Case I (separable clusters). 

Note that in this case, we have no high probability guarantees on the cluster assignment by  CMLB. However, in this situation also, the regret suffered by the users are not very large. This is because the maximum separation between the clusters (and hence the clustering error) is  $\max_{j \in \mathcal{S}} \| \theta^*_i - \theta^*_j \| \leq 5/(NT)^\alpha $. Furthermore, since we have no control on how many clusters CMLB club, in the worst case, the minimum cluster size will be $p^*N$ (the input size parameter to the CMLB subroutine). Hence, we obtain
\begin{align*}
R(coll-learn) \leq C \sqrt{\frac{d}{\rho_{\min}} } \,\, \left(\sqrt{\frac{T-\frac{d (NT)^{2\alpha}}{\rho_{\min}} \log(1/\delta)}{p^* N}} \right) \sqrt{\log(1/\delta)}
\end{align*}
with probability at least $1-\delta$.

We now characterize the regret from the cluster miss-specification. This term occurs since $\Delta \leq 5/(NT)^\alpha$. For agent $i$, from the OFUL algorithm (see \citep{chatterji2020osom,oful}, we see that the regret in linearly dependent on $\theta^*_1$.

Hence, following the regret analysis of stochastic linear bandits (see \citep{chatterji2020osom}, using triangle inequality, and the condition $\max_{i \neq j} \|\theta^*_i - \theta^*_j \| \leq 5L/(NT)^\alpha$), we have
\begin{align*}
    R(cluster-error) \leq C L \left(  \frac{T - \frac{d (NT)^{2\alpha}}{\rho_{\min}} \log(1/\delta)}{(NT)^{\alpha}}\right).
\end{align*}

Now, combining all $3$ components, we have
\begin{align*}
    R(T) &\leq C_1 \,\, \frac{d}{\rho_{\min}}  (NT)^\alpha \,\, \log(1/\delta) + C_2 \sqrt{\frac{d}{\rho_{\min}} } \,\, \left(\sqrt{\frac{T-\frac{d (NT)^{2\alpha}}{\rho_{\min}} \log(1/\delta)}{p^* N}} \right) \sqrt{\log(1/\delta)} \\
    & + C_3 L \left(  \frac{T - \frac{d (NT)^{2\alpha}}{\rho_{\min}} \log(1/\delta)}{(NT)^{\alpha}}\right).
\end{align*}
Rewriting, we have
\begin{align*}
    R(T) &\leq C_2 \sqrt{\frac{d}{\rho_{\min}} } \,\, \left(\sqrt{\frac{T-\frac{d (NT)^{2\alpha}}{\rho_{\min}} \log(1/\delta)}{p^* N}} \right) \sqrt{\log(1/\delta)} \\
    &+ C_3 L \left(\frac{T}{(NT)^\alpha}\right) + (C_1 - C_3 L) \left(  \frac{T - \frac{d (NT)^{2\alpha}}{\rho_{\min}} \log(1/\delta)}{(NT)^{\alpha}}\right).
\end{align*}
Since $L\geq 1$, choosing $C_3 > C_1$, we obtain
\begin{align*}
    R(T) \leq C_2 \sqrt{\frac{d}{\rho_{\min}} } \,\, \left(\sqrt{\frac{T-\frac{d (NT)^{2\alpha}}{\rho_{\min}} \log(1/\delta)}{p^* N}} \right) \sqrt{\log(1/\delta)}
    + C_3 \left(\frac{T}{(NT)^\alpha}\right).
    \end{align*}
Now, suppose $\alpha = \frac{1}{2}- \varepsilon$, where $\varepsilon$ is a positive constant arbitrarily close to $0$. In that case, we obtain
\begin{align*}
    R(T) &\leq C \left[ L \left( \frac{T^{\frac{1}{2}+\varepsilon}}{N^{1-\varepsilon}} \right) + \sqrt{\frac{d}{\rho_{\min}} } \,\, \left(\sqrt{\frac{T-\frac{d (NT)^{2\alpha}}{\rho_{\min}} \log(1/\delta)}{p^* N}} \right) \sqrt{\log(1/\delta)} \right] \\
    & \leq C \left[ L \left( \frac{T^{\frac{1}{2}+\varepsilon}}{N^{1-\varepsilon}} \right) + \sqrt{\frac{d}{\rho_{\min}} } \,\, \left(\sqrt{\frac{T}{p^* N}} \right) \sqrt{\log(1/\delta)} \right]
\end{align*}
with probability at least $1- 4 \binom{N}{2} \delta$.

\section{Special Case: all clusters are close}
\label{sec:all_club}
 Let us consider the pairwise differences this setting where we treat all the agents as one big cluster. Without loss of generality, we focus on $\{\|\Hat{\theta}^{(i)} - \Hat{\theta}^{(j)}\|\}_{j\neq i}$, and assume that the agent  belongs to cluster 1 with parameter $\theta^*_1$. In this setting, if the $j$-th agent falls in cluster 1, we have
\begin{align*}
    \|\Hat{\theta}^{(j)} - \Hat{\theta}^{(i)}\| \leq 2/(NT)^\alpha.
\end{align*}
with probability at least $1-2\delta$. Otherwise, we obtain
\begin{align*}
    \|\Hat{\theta}^{(j)} - \Hat{\theta}^{(i)}\| &\leq 2/(NT)^\alpha + \max_{i,j}\|\theta^*_i - \theta^*_j \| \\
    & \leq 3/(NT)^\alpha
    \end{align*}
with probability exceeding $1-2\delta$. 
Ignoring the constants for now, if we have
\begin{align*}
    \max_{j\neq i}\{\|\Hat{\theta}^{(i)} - \Hat{\theta}^{(j)}\|\} \leq 3/(NT)^\alpha,
\end{align*}
then with probability at least $1-2 \binom{N}{2} \delta$, and everyone belongs to the same cluster. So, we can put the threshold as $3/(NT)^\alpha$ to identify whether there is a cluster structure present or not.

The regret computation in this setup follows from Lemma~\ref{lem:inseparable_clusters}, with $2$ differences: \\
(a) The clustering error is $\max_{i,j}\|\theta^*_i - \theta^*_j \| 
     \leq 1/(NT)^\alpha$. Note that order-wise it is same as the misclustering error for Lemma~\ref{lem:inseparable_clusters}.\\
(b) Here, since all the clusters are close and CMLB puts everyone in the same cluster, the collaborative learning gain will be $\sqrt{\frac{1}{N}}$. Hence, the regret bound follows from Theorem~\ref{thm:inseparable} with these modifications.

\section{Analysis of SCLB in Algorithm \ref{algo:size_unknown}}
\textbf{Proof of Theorem \ref{thm:sclb_separated}: Minimum Cluster size is larger than $\frac{5}{(NT)^{\alpha}}$}

We give the proof in the case when $\Delta_i > \frac{5}{(NT)^{\alpha}}$. The other setting follows identically.
In each phase $i$, the size parameter used in Line $3$ of Algorithm \ref{algo:size_unknown} is $i^{-2}$. Thus for all phases $ i\geq \lceil \frac{1}{\sqrt{p}} \rceil$, the input size parameter to Algorithm \ref{algo:size_unknown} invoked in Line $3$ of Algorithm \ref{algo:size_unknown} is correct, i.e., $i^{-2} \leq p$, and thus satisfies the conditions of Lemma \ref{lem:separated_cluster}. 

In the rest of the proof, denote by $i^{*} := \lceil \frac{1}{\sqrt{p}} \rceil$ and by $T_i = 2^i$, for all $i \geq 1$. Lemma \ref{lem:inseparable_clusters} states that, for any  $ i\geq i^{*}$, the regret incurred by any agent in phase $i$ satisfies
\begin{align}
        R(T_i) &\leq C_1  \left[\frac{d}{\rho_{\min}}  (NT_i)^\alpha + \sqrt{\frac{d}{\rho_{\min}} } \,\, \left(\sqrt{\frac{T_i-\frac{d (NT_i)^{2\alpha}}{\rho_{\min}} \log(2^i/\delta)}{i^{-2} N}} \right) \right] \log(2^i/\delta),
        \label{eqn:prf_unknwn_phase}
\end{align}
with probability at-least $1 - 4 {N \choose 2}2^{-i} \delta$. We now use a simple regret decomposition and an union bound to conclude the proof of Theorem \ref{thm:sclb_separated}. 

Observe that, in a time horizon of $T$, there are at-most $\lceil \log_2(T) \rceil$ number of phases. The total regret can be decomposed as 
\begin{align*}
    R(T) &\leq \sum_{i=1}^{\lceil \log_2(T) \rceil}R(T_i),\\
    &\leq \sum_{i=1}^{i^{*}}2^i + \sum_{i=i^{*}}^{\lceil \log_2(T) \rceil}R(T_i),\\
    &\leq 2^{i^{*}+1} + \sum_{i=i^{*}}^{\lceil \log_2(T) \rceil}R(T_i).
\end{align*}
In the first equality, we upper bound by assuming that the agent incurs a regret of $1$, in all time steps till phase $i^{*}$. Now from an union bound, we can conclude that with probability at-least $1 - \sum_{i=1}^{\lceil \log_2(T) \rceil}cN^22^{-i}\delta \geq 1-2cN^2\delta$, Equation (\ref{eqn:prf_unknwn_phase}) is satisfied for all $i \geq i^{*}$. 

Combining the above facts, along with the definition that $T_i = 2^i$, we have, 
\begin{align*}
    R(T_i) &\leq 2^{i^{*}+1} +  C \log \left( \frac{2}{\delta} \right) \bigg[  \sum_{i=1}^{\lceil \log_2(T) \rceil} \bigg(  \frac{dN^{\alpha}}{\rho_{min}} i2^{i\alpha} + \sqrt{\frac{d}{\rho_{min}{N}}} i^2\sqrt{2^{i}} \bigg) \bigg],\\
    & \leq 2^{i^{*}+1} +  C_1 \log \left( \frac{2}{\delta} \right) \log^2(T) \bigg[  \sum_{i=1}^{\lceil \log_2(T) \rceil} \bigg(  \frac{dN^{\alpha}}{\rho_{min}} 2^{i\alpha} + \sqrt{\frac{d}{\rho_{min}{N}}} \sqrt{2^{i}} \bigg) \bigg]
\end{align*}
with probability at-least $1-2cN^2\delta$. The second inequality follows by upper bounding $i \leq \lceil \log_2(T) \rceil$. Observe from the definition of $i^{*}$ that $2^{i^{*}+1} \leq 4\left(2^{\frac{1}{\sqrt{p}}} \right)$, the regret is bounded by 
\begin{align*}
    R(T) \leq 4\left(2^{\frac{1}{\sqrt{p}}} \right) +  C_2 \log \left( \frac{2}{\delta} \right) \log^2(T) \bigg[  (NT)^{\alpha}  \frac{d}{\rho_{min}}  + \sqrt{T\frac{d}{\rho_{min}{N}}} \bigg].
\end{align*}
Here $C_1, C_2$ are universal constants. 

\textbf{Theorem \ref{thm:inseparable}: Minimum Cluster size is smaller than or equal to $\frac{5}{(NT)^{\alpha}}$}

Following the identical steps as for Case $I$, where we use Lemma \ref{lem:inseparable_clusters} to bound the regret in a phase, we get that with probability at-least $1-2cN^2 \delta$
\begin{align*}
    R(T) \leq C_1L \log^2(T) 2^{1-\alpha}\frac{T^{1-\alpha}}{N^{\alpha}} + C_2 \log^2(T) \sqrt{T\frac{d}{N\rho_{min}}}\log(2/\delta).
\end{align*}

\section{Proof of Theorem~\ref{thm:personal}}
\paragraph{Collaboratively learn the common representations:}
Recall the setup of collaborative learning; at each time $t$, out of $K$ contexts available at the center, $\{\beta_{r,t}\}_{r=1}^K$, the center chooses a context vector, (call it $\beta_{r,t}$, corresponding to the $r$-th arm), and broadcasts to all the agents. Agent $i$, using the context $\beta_{r,t}$, observes the following reward:
\begin{align*}
    y^{(i)}_t = \langle \beta_{r,t},\theta^*_i \rangle + \eta_{i,t},
\end{align*}
and sends this to the center. Similarly, all the $N$ agents observes their reward and send those to the center. The center then averages this rewards and obtain
\begin{align*}
    \frac{1}{N} \sum_{l=1}^N y_t^{(l)} = \langle \beta_{r,t}, \frac{1}{N} \sum_{l=1}^N \theta^*_l \rangle + \frac{1}{N}\sum_{l=1}^N \eta_{l,t}.
\end{align*} 

Since, we are averaging i.i.d noise, the variance decreases by a factor of $N$. Now, based on the average reward, the center choosing the next arm by playing the stochastic contextual Bandit algorithm OFUL. Hence, in this phase, the center indeed learns the parameter $\Bar{\theta}^*:= \frac{1}{N} \sum_{l=1}^N \theta^*_l$. We let this phase run for $\sqrt{T}$ rounds, and let $\Hat{\theta}^*$ be the corresponding estimate. Provided, $T > \tau_{\min}^2(\delta)$, from \citep{chatterji2020osom}, we have,
\begin{align*}
    \|\Hat{\theta}^* - \Bar{\theta}^*\| \leq \Tilde{O}\left( \sqrt{\frac{d}{\rho_{\min} N \sqrt{T}}} \right) \log(1/\delta),
\end{align*}
with probability at least $1-\delta$. The corresponding regret (call it $R_{c,1}$) is
\begin{align*}
    R_{c,1} = \Tilde{O}\left( \sqrt{\frac{d \sqrt{T}}{\rho_{\min} N}} \right) \log(1/\delta),
\end{align*}
with probability at least $1-\delta$.

Note that additional to the above, we incur a regret since instead of learning $\theta^*_i$, we are actually learning $\Bar{\theta}^*$. This is equivalent to clustering with miss-specification. Following the proofs similar to Section~\ref{sec:clustering}, we obtain
\begin{align*}
    R_{c,2} = \|\theta^*_i - \Bar{\theta}^*\| \, T_1 \leq \epsilon_i\, \sqrt{T},
\end{align*}
where we use the fact that $\|\theta^*_l\| \leq 1$ for all $l \in [N]$. Hence, the total regret in this phase is
\begin{align*}
    R_{c,1} + R_{c,2} \leq  \Tilde{O}\left( \sqrt{\frac{d \sqrt{T}}{\rho_{\min} N}} \right) \log(1/\delta) + \epsilon_i\, \sqrt{T},
\end{align*}
with probability at least $1-\delta$.

\paragraph{Personal Learning:}

At each time $t$, out of $K$ contexts available at the center, $\{\beta_{r,t}\}_{r=1}^K$, suppose the center chooses a context vector, $\beta_{r,t}$, (corresponding to the $r$-th arm) and recommends it to agent $i$. Thereafter, agent $i$ generates the reward $y_t^{(i)} = \langle \beta_{r,t},\theta^*_i \rangle + \xi_{i,t}$, and sends it to the center. Subsequently, the center calculates the corrected reward
\begin{align*}
    \Tilde{y}_t^{(i)} = y_t^{(i)} - \langle \beta_{r,t},\Hat{\theta}^* \rangle.
\end{align*}
Note that the center has the information about $(\beta_{r,t},\Hat{\theta}^*)$ and so it can compute $\Tilde{y}_t^{(i)}$. With this shift, the center basically learns the vector $\theta^*_i - \Hat{\theta}^*$. 

In this phase we use the \texttt{ALB-norm} algorithm of \citep{ghosh_adaptive}\footnote{In particular we use  \texttt{ALB-norm} algorithm of \citep{ghosh_adaptive}, with $\tau = \mathcal{O}(1)$, and zero arm biases, and hence no pure exploration to estimate the arm biases.}. Note that the \texttt{ALB-norm} algorithm is a norm adaptive algorithm, which is particularly useful when the parameter norm is small. \texttt{ALB-norm} uses the OFUL algorithm of \citep{chatterji2020osom,oful} repeatedly over epochs. At the beginning of each epoch, it estimates the parameter norm, and runs OFUL with the norm estimate (see \citep[Algorithm 1]{ghosh_adaptive}). Hence, it is shown in \citep[Algorithm 1]{ghosh_adaptive} that while estimating the parameter $\Psi^*$, with high probability, the regret of \texttt{ALB-norm} is
\begin{align*}
    R_{\texttt{ALB-norm}} \leq \|\Psi^*\| \,\, R_{OFUL}.
\end{align*}

In Appendix~\ref{sec:shift-oful}, we present an analysis of shifted OFUL. In particular we show that shifts (by a fixed vector) can not reduce the regret (which is intuitive). Note that we learn $\Hat{\theta}^*$ in the common learning phase, and fix it throughout the personal learning phase. Hence, conditioned on the observations of the common learning phase, $\Hat{\theta}^*$ is a fixed (deterministic) vector. Also,
\begin{align*}
    \|\theta^*_i - \Hat{\theta}^*\| &\leq \|\theta^*_i - \Bar{\theta}^*\| + \| \Hat{\theta}^* - \Bar{\theta}^*\| \\
    & \leq \epsilon_i + \Tilde{O}\left( \sqrt{\frac{d}{\rho_{\min} N \sqrt{T}}} \right) \log(1/\delta),
\end{align*}
with probability at least $1-\delta$. Hence, using Lemma~\ref{lem:shift-oful} of Appendix~\ref{sec:shift-oful}, the regret in the personal learning phase (call it $R_i^{(p)}$) is given by
\begin{align*}
    R_i^{(p)}  \leq \Tilde{\mathcal{O}}\left( \|\theta^*_i - \Hat{\theta}^*\| \sqrt{d (T - \sqrt{T})} \log(1/\delta) \right),
\end{align*}
with probability at least $1- c\delta - \frac{1}{\poly(T)}$, provided $d \geq C \log(K^2 T)$. Substituting, we obtain
\begin{align*}
    R_i^{(p)} &\leq \Tilde{\mathcal{O}} \left ( \left[\epsilon_i + \Tilde{O}\left( \sqrt{\frac{d}{\rho_{\min} N \sqrt{T}}} \right) \log(1/\delta) \right] \sqrt{d (T - \sqrt{T})} \log(1/\delta) \right) \\
    & \leq \Tilde{\mathcal{O}} \left ( \left[\epsilon_i +  \sqrt{\frac{d}{\rho_{\min} N \sqrt{T}}}  \right] \sqrt{d T}  \right) \log^2(1/\delta),
\end{align*}
with probability exceeding $1-c\delta- \frac{1}{\poly(T)}$.

\paragraph{Total Regret:} 
We now characterize the total regret of agent $i$. We have
\begin{align*}
   & R_i(T) = R_i^{(c,1)} + R_i^{(c,2)} + R_i^{(p)} \\
   &\leq \Tilde{O}\left( \sqrt{\frac{d \sqrt{T}}{\rho_{\min} N}} \right) \log(1/\delta) + \epsilon_i\, \sqrt{T} + \Tilde{\mathcal{O}} \left ( \left[\epsilon_i +  \sqrt{\frac{d}{\rho_{\min} N \sqrt{T}}}  \right] \sqrt{d T}  \right) \log^2(1/\delta) \\
   & \leq \Tilde{\mathcal{O}} \left[ \left(\epsilon_i\, \sqrt{T} + \epsilon_i\, \sqrt{dT} \right) +  T^{1/4} \left(\sqrt{\frac{d}{\rho_{\min} N}} + \sqrt{\frac{d^2}{\rho_{\min} N}} \right) \right] \log^2(1/\delta) \\
   & \leq \Tilde{\mathcal{O}} \left[ \epsilon_i\, \sqrt{dT}  + \,\, T^{1/4}\,\, \sqrt{\frac{d^2}{\rho_{\min} N}}\right] \log^2(1/\delta)
\end{align*}
with probability at least $1-c_1 \delta - \frac{1}{\poly(T)}$.

\section{Proof of Corollary~\ref{cor:exp-regret}}
In order to obtain the expected regret, one writes expectation as an integral of the tail probabilities and use the high probability bound to compute the tail probability. With this, in the common learning phase, we have
\begin{align*}
    \E R^{(c,1)}_i \leq \Tilde{O}\left( \sqrt{\frac{d \sqrt{T}}{\rho_{\min} N}} \right),
\end{align*}
and 
\begin{align*}
    \E R^{(c,2)}_i  \leq \epsilon_i\, \sqrt{T}.
\end{align*}
In the personal learning phase we use Corollary~\ref{cor:zero-mean} of Appendix~\ref{sec:shift-oful}, which says that shifting makes the regret worse in expectation. Hence, using \citep[Theorem 1]{ghosh_adaptive} and converting it to an expected regret, we have
\begin{align*}
    \E R_i^{(p)} \leq \Tilde{\mathcal{O}} \left ( \left[\epsilon_i +  \sqrt{\frac{d}{\rho_{\min} N \sqrt{T}}}  \right] \sqrt{d T}  \right).
\end{align*}
The final regret bound follows from summing up the above 3 expressions.

\section{Shifted OFUL Regret}
\label{sec:shift-oful}

In this section, we want to establish a relationship between the regret of the standard OFUL algorithm and the shift compensated algorithm. We define the shifted version of OFUL below.

\begin{definition}
The OFUL algortihm is used to make a decision of which action to take at time-step $t$, given the history of past actions $X_1, \cdots, X_{t-1}$ and observed rewards $Y_1, \cdots, Y_{t-1}$. The $\Gamma$ shifted OFUL is an algorithm identical to OFUL that describes the action to take at time step $t$, based on the past actions $X_1, \cdots, X_{t-1}$ and the observed rewards $\widetilde{Y}_1^{(\Gamma)}, \cdots, \widetilde{Y}_{t-1}^{(\Gamma)}$, where for all $1 \leq s \leq t-1$, $\widetilde{Y}_s = Y_s - \langle X_s, \Gamma \rangle$. 
\end{definition}

\begin{definition}
For a linear bandit instance with unknown parameter $\theta^*$, and a sequence of (possibly random) actions $X_{1:T} := X_1, \cdots, X_T$, denote by $\mathcal{R}_T(X_{1:T}) := \sum_{t=1}^T  \max_{1 \leq j \leq K} \langle \beta_{j,t} - X_t, \theta^* \rangle$.
\end{definition}

\begin{definition}
For a linear bandit system with unknown parameter $\theta^*$, and a sequence of (possibly random) actions $X_{1:T}:=X_1, \cdots, X_T$, denote by $\mathcal{R}_{T}^{(\Gamma)}(X_{1:T}) := \sum_{t=1}^T \max_{1 \leq j \leq K} \langle \beta_{j,t} - X_t, \theta^* - \Gamma \rangle$.
\end{definition}

\begin{proposition}
Suppose for a linear bandit instance with parameter $\theta^*$, an algorithm plays the sequence of actions $X_1, \cdots, X_T$, then
\begin{align*}
    \mathcal{R}_T(X_{1:T}) \leq \mathcal{R}_{T}^{(\Gamma)}(X_{1:T}) + \sum_{t=1}^T \left( \langle X_t -  \argmax_{\beta \in \{ \beta_{1,t}, \cdots, \beta_{K,t} \}} \langle \beta , \theta^* \rangle, \Gamma \rangle \right).
\end{align*}
\label{prop:shift_regret}
\end{proposition}
\begin{proof}
From the definition of $\mathcal{R}_{T}^{(\Gamma)}$, we can write the regret as 
\begin{align}
    \mathcal{R}_{T}^{(\Gamma)}(X_{1:T}) &=  \sum_{t=1}^T  \max_{1 \leq j \leq K} \langle \beta_{j,t} - X_t, \theta^* + \Gamma \rangle, \nonumber \\
    & \stackrel{(a)}{\leq} \sum_{t=1}^T \max_{1 \leq j \leq K} \langle \beta_{j,t} , \theta^* \rangle +  \langle \beta^*_t, \Gamma \rangle - \langle X_t, \theta^* \rangle - \langle X_t, \Gamma \rangle, \label{eqn:regret_shift_decomp}
\end{align}
where, $\beta^*_t := \argmax_{\beta \in \{\beta_{1,t}, \cdots, \beta_{K,t} \}} \langle \beta, \theta^* \rangle$. The inequality $(a)$ follows from the following elementary fact.
\begin{lemma}
Let $\mathcal{X}$ be a compact set, and functions $f,g : \mathcal{X} \rightarrow \mathbb{R}$, such that $\sup_{x \in \mathcal{X}} |f(x)| < \infty$ and $\sup_{x \in \mathcal{X}}|g(x)| < \infty$. Then, 
$$\max_{x \in \mathcal{X}}(f(x) + g(x)) \geq \max_{x \in \mathcal{X}} f(x) + \min_{x \in \mathcal{X}} g(x).$$
\end{lemma}
that Rewriting Equation (\ref{eqn:regret_shift_decomp}), we see that 
\begin{align*}
    \mathcal{R}_{T}^{(\Gamma)}(X_{1:T}) \leq \mathcal{R}_T + \sum_{t=1}^T \langle \beta^*_t - X_t, \Gamma \rangle,
\end{align*}
and thus the proposition is proved.
\end{proof}

\begin{corollary}
\label{cor:zero-mean}
If for every time $t \geq 1$, the set of $K$ context vectors $\beta_{1,t}, \cdots, \beta_{K,t}$ are all $0$ mean random variables, then
\begin{align*}
    \mathbb{E}[\mathcal{R}_T] \leq \mathbb{E}[ \mathcal{R}_{T}^{(\Gamma)}].
\end{align*}
\end{corollary}

\begin{corollary}
Suppose for all time $t$, $\mathrm{\argmax}_{\beta \in \{ \beta_{1,t}, \cdots, \beta_{K,t} \}} \langle \beta , \theta^* \rangle = \argmax_{\beta \in \{ \beta_{1,t}, \cdots, \beta_{K,t} \}} \langle \beta , \Gamma \rangle$. Then, 
\begin{align*}
    \mathcal{R}_T(X_{1:T}) \leq \mathcal{R}_{T}^{(\Gamma)}(X_{1:T}).
\end{align*}
\end{corollary}
\begin{proof}
From the hypothesis of the theorem, we can observe the following, 
\begin{align*}
    \sum_{t=1}^T \left( \langle X_t -  \argmax_{\beta \in \{ \beta_{1,t}, \cdots, \beta_{K,t} \}} \langle \beta , \theta^* \rangle, \Gamma \rangle \right) &= \sum_{t=1}^T \left( \langle X_t -  \argmax_{\beta \in \{ \beta_{1,t}, \cdots, \beta_{K,t} \}} \langle \beta , \Gamma \rangle, \Gamma \rangle \right), \\
    &\leq 0.
\end{align*}
Plugging the above bound into Proposition \ref{prop:shift_regret} completes the proof. 
\end{proof}

\subsubsection{ High Probability Bound on $\mathcal{R}_T^{(\Gamma)}$ }

\begin{lemma}
Suppose the $K$ context vectors $\beta_1, \cdots, \beta_K$ are such that for all $i$, $||\beta_i|| \leq 2$ and for all $i \neq j$, $ | \langle \beta_i - \beta_j , \theta^* \rangle |  \geq 4  || \theta^* - \Gamma||$, where $\theta^*$ is the unknown linear bandit parameter and $\Gamma$ is a fixed vector. Then
\begin{align*}
    \argmax_{1\leq j \leq K} \langle \beta_j, {\theta}^*_i \rangle = \argmax_{1\leq j \leq K} \langle \beta_j, \Gamma \rangle.
\end{align*}
\label{lem:when_are_optimal_equal}
\end{lemma}
\begin{proof}
We will prove the following more stronger statement. Let $i \neq j \in [K]$ be such that $\langle \theta^*, \beta_i \rangle \geq \langle \theta^*, \beta_j \rangle$. Then, under the hypothesis of the proposition statement, we have $\langle \theta^*, \beta_i - \beta_j \rangle \geq 4 || \theta^* - \Gamma ||$. Thus, the following chain holds, 
\begin{align*}
    \langle \beta_i - \beta_j, \Gamma \rangle &= \langle  \beta_i - \beta_j, \theta^*\rangle + \langle  \beta_i - \beta_j, \Gamma - \theta^* \rangle, \\
    &\geq 4 || \theta^* - \Gamma || + \langle  \beta_i - \beta_j, \Gamma - \theta^* \rangle, \\
    &\geq 4  || \theta^* - \Gamma || - ||\beta_i-\beta_j||||\Gamma - \theta^*||, \\
    &\geq 0.
\end{align*}
The first inequality follows from the hypothesis of the proposition statement, the second follows from Cauchy Schwartz inequality and the last follows from the fact that $||\beta_i - \beta_j|| \leq 2$. Thus, we have shown that under the hypothesis of the Proposition, the ordering of the coordinates whether by inner product with $\theta^*$ or with $\Gamma$ remains unchanged. In particular, the argmax is identical.
\end{proof}

\begin{lemma}
Let $\theta^*$ be a fixed vector with $\|\theta^*\| \leq 1$, and $\Gamma \in \mathbb{R}^d$ be any arbitrary vector such that $|| \theta^* - \Gamma|| \leq \psi$, for some constant $\psi$. Let $\beta_1, \cdots, \beta_K$ be i.i.d. vectors, each distributed as $\mathsf{Unif}[-c/\sqrt{d},c/\sqrt{d}]^{\otimes d}$ for a constant $c$. Then, 
\begin{align*}
    \mathbb{P} \left[     \mathrm{\argmax}_{1\leq j \leq K} \langle \beta_j, {\theta}^*_i \rangle = \argmax_{1\leq j \leq K} \langle \beta_j, \Gamma \rangle \right] \geq \left( 1- {K \choose 2} e^{-\frac{d}{4}(1-8\psi^2)^2} - K e^{-\frac{\sqrt{5}-1}{2}d} \right).
    \end{align*}
\label{lem:anti_concentration}
\end{lemma}

\begin{proof}
Denote by the \emph{Good event} $\mathcal{E} := \left\{ \argmax_{1\leq j \leq K} \langle \beta_j, {\theta}^*_i \rangle = \argmax_{1\leq j \leq K} \langle \beta_j, \Gamma \rangle \right\}$
From Lemma \ref{lem:when_are_optimal_equal}, we know that a sufficient condition for event $\mathcal{E}$ to hold is that for all $i \neq j$, we have $\bigg|\langle \theta^*, \beta_i - \beta_j \rangle\bigg| \geq 2 || \theta^* - \Gamma ||$ and for all $i$, $||\beta_i|| < 2$. Thus, from a simple union bound, we get
\begin{align*}
    \mathbb{P}[\mathcal{E}^c] &\leq \sum_{1 \leq i < j \leq K} \mathbb{P} \left[\bigg|\langle \theta^*, \beta_i - \beta_j \rangle\bigg| \leq 4  || \theta^* - \Gamma ||\right] + \sum_{i=1}^K \mathbb{P}[||\beta_i|| \geq 2], \\
    & = {K \choose 2} \mathbb{P} \left[\bigg|\langle \theta^*, \beta_1 - \beta_2 \rangle\bigg| \leq 4 || \theta^* - \Gamma ||\right] + K \mathbb{P}[||\beta_1|| \geq 2].
\end{align*}
The second equality follows from the fact that $\beta_1, \cdots, \beta_K$ are i.i.d. Now, since $||\theta^*|| \leq 1$, we have from Cauchy Schwartz that, almost-surely, $\bigg|\langle \theta^*, \beta_1 - \beta_2 \rangle\bigg|  \leq || \beta_1 - \beta_2 ||$. Thus, 
\begin{align*}
    \mathbb{P} \left[|\langle \theta^*, \beta_1 - \beta_2 \rangle | \leq 4 || \theta^* - \Gamma || \right] &\leq \mathbb{P} [ || \beta_1 - \beta_2 ||  \leq 4  || \theta^* - \Gamma ||], \\
    & \leq  \mathbb{P} [ || \beta_1 - \beta_2 || \leq 4 \psi ], \\
    &= \mathbb{P} [ || \beta_1 - \beta_2 ||^2 \leq 16 \psi^2 ], \\
    &\stackrel{(a)}{\leq} e^{-\frac{c_1 d}{4}},
\end{align*}
where the constant $c_1$ depends on $\psi$. The first inequality follows from Cauchy Schwartz, and the fact that $|| \theta^* || \leq 1$. The last inequality follows from the fact that, $\mathbb{E}\| \beta_1 - \beta_2\|^2 = c_2$ for a constant $c_2$, and since $\{\beta_1,\beta_2\}$ are coordinate-wise bounded, we use standard sub-Gaussian concentration to argue that $\|\beta_1 - \beta_2\|^2$ is close to its expectation. Finally, we obtain that 
\begin{align*}
    \mathbb{P}\left( \|\beta_1 - \beta_2\|^2 - \mathbb{E}\|\beta_1 - \beta_2\|^2 \leq -t \right) \leq \exp \left ( - c_3 \,d t^2  \right).
\end{align*}
Choosing $t$ as a constant, we obtain (a).

Finally, we also need to ensure that the context vectors $\beta_1, \cdots, \beta_K$ have norms bounded by $2$. This can also be similarly be bounded by the upper tail inequality as
    \begin{align*}
        \mathbb{P} [ || \beta_1|| \geq 2 ] &= \mathbb{P}[ d || \beta_1||^2 \geq 4 d], \\
        &\stackrel{(b)}{\leq} e^{- c_4 d}. 
    \end{align*}
for a constant $c_4$, where inequality $(b)$ follows from the upper-tail concentration bound for sub-Gaussian random variables. Putting this all together concludes the proof.
\end{proof}

\begin{lemma}
\label{lem:shift-oful}
Consider a linear bandit instance with parameter $\theta^*$ with $||\theta^*|| \leq  1$ and the context vectors at each time are sampled uniformly and independently from the scaled normal distribution in $d$ dimensions, i.e., the contexts are i.i.d. across time and arms from $\mathsf{Unif} [-1/\sqrt{d},1/\sqrt{d}]^{\otimes d}$. Let $\Gamma \in \mathbb{R}^d$ be such that $|| \theta^* - \Gamma || \leq \psi $ with $\psi < \frac{1}{2\sqrt{2}}$, and $X_{1:T} = (X_1, \cdots, X_T)$ be the set of actions chosen by the $\Gamma$ shifted OFUL. Then, with probability at-least $\left( 1- {K \choose 2} e^{-\frac{d}{4}(1-8\psi^2)^2} - Ke^{-\frac{\sqrt{5}-1}{2}d} \right)$,
\begin{align*}
    \mathcal{R}_T(X_{1:T}) \leq \mathcal{R}^{(\Gamma)}_T(X_{1:T}).
\end{align*}
\label{Lem:shift_equal_true_regret_whp}
\end{lemma}
\begin{proof}
This follows by combining Lemma \ref{lem:anti_concentration} and \ref{lem:when_are_optimal_equal}.
\end{proof}
\vfill

\end{document}